%% file: main.tex
\documentclass{article} % For LaTeX2e
\usepackage{iclr2021_conference_notitle,times}

\iclrfinalcopy

\usepackage{wrapfig}

\usepackage{algorithm}
\usepackage[noend]{algpseudocode}
\usepackage{algorithmicx}
\usepackage{amsmath,amssymb,amsfonts,amsthm}

\algblock{ParFor}{EndParFor}
\algnewcommand\algorithmicparfor{\textbf{parfor}}
\algnewcommand\algorithmicpardo{\textbf{do}}
\algnewcommand\algorithmicendparfor{\textbf{end\ parfor}}
\algrenewtext{ParFor}[1]{\algorithmicparfor\ #1\ \algorithmicpardo}
\algrenewtext{EndParFor}{\algorithmicendparfor}

\usepackage[utf8]{inputenc} % allow utf-8 input
\usepackage[T1]{fontenc}    % use 8-bit T1 fonts
\usepackage[colorlinks=true, citecolor=blue]{hyperref}
\usepackage{url}            % simple URL typesetting
\usepackage{booktabs}       % professional-quality tables
\usepackage{amsfonts}       % blackboard math symbols
\usepackage{nicefrac}       % compact symbols for 1/2, etc.
\usepackage{microtype}      % microtypography

\usepackage{hyperref}       % hyperlinks
\usepackage{url}            % simple URL typesetting
\usepackage{booktabs}       % professional-quality tables
\usepackage{nicefrac}       % compact symbols for 1/2, etc.
\usepackage{amssymb} % empty set symbol
\usepackage{amsmath}
\usepackage{amsthm}
\usepackage{graphicx}
\usepackage{subcaption}
\usepackage{caption}
\usepackage{color}
\usepackage{mathtools}
\usepackage[symbol]{footmisc}
\usepackage{tablefootnote}

\theoremstyle{definition}

\theoremstyle{definition}
\newtheorem{proposition}{Proposition}

\newcommand{\norm}[1]{\left\lVert#1\right\rVert}

\usepackage{footnote}

\title{Adversarial score matching and improved sampling for image generation}

\begin{document}

\maketitle
\input{sections/authors_hardcoded.tex}

\begin{abstract}
Denoising Score Matching with Annealed Langevin Sampling (DSM-ALS) has recently found success in generative modeling. The approach works by first training a neural network to estimate the score of a distribution, and then using Langevin dynamics to sample from the data distribution assumed by the score network. Despite the convincing visual quality of samples, this method appears to perform worse than Generative Adversarial Networks (GANs) under the Fréchet Inception Distance, a standard metric for generative models. We show that this apparent gap vanishes when \em denoising \em the final Langevin samples using the score network.
In addition, we propose two improvements to DSM-ALS:  1) Consistent Annealed Sampling as a more stable alternative to Annealed Langevin Sampling, and 2) a hybrid training formulation, composed of both Denoising Score Matching and adversarial objectives. By combining these two techniques and exploring different network architectures, we elevate score matching methods and obtain results competitive with state-of-the-art image generation on CIFAR-10.

\end{abstract}

\section{Introduction}

\citet{song2019generative} recently proposed a novel method of generating samples from a target distribution through a combination of Denoising Score Matching (DSM) \citep{hyvarinen2005estimation, vincent2011connection, raphan2011least} and Annealed Langevin Sampling (ALS) \citep{welling2011bayesian,roberts1996exponential}. 
Since convergence to the distribution is guaranteed by the ALS, their approach (DSM-ALS) produces high-quality samples and guarantees high diversity. Though, this comes at the cost of requiring an iterative process during sampling, contrary to other generative methods.
%The main benefits of their approach (DSM-ALS) are that it produces relatively high-quality samples and guarantees high diversity due to ALS. The main drawback of DSM-ALS is that generating samples requires many iterations rather than being done in one-shot.

\citet{song2020improved} further improved their approach by increasing the stability of score matching training and proposing theoretically sound choices of hyperparameters. They also scaled their approach to higher-resolution images and showed that DSM-ALS is competitive with other generative models.
\citet{song2020improved} observed that the images produced by their improved model were more visually appealing than the ones from their original work;
however, the reported Fréchet Inception Distance (FID) \citep{heusel2017gans} did not correlate with this improvement. 

Although DSM-ALS is gaining traction, Generative adversarial networks (GANs) \citep{GAN} remain the leading approach to generative modeling. GANs are a very popular class of generative models; they have been successfully applied to image generation \citep{brock2018large,karras2017progressive,karras2019style,karras2020analyzing} and have subsequently spawned a wealth of variants \citep{radford2015unsupervised,miyato2018spectral,jolicoeur2018relativistic,zhang2019self}. The idea behind this method is to train a Discriminator ($D$) to correctly distinguish real samples from fake samples generated by a second agent, known as the Generator ($G$). GANs excel at generating high-quality samples as the discriminator captures features that make an image plausible, while the generator learns to emulate them.

Still, GANs often have trouble producing data from all possible modes, which limits the diversity of the generated samples. A wide variety of tricks have been developed to address this issue in GANs \citep{kodali2017convergence,WGAN-GP,WGAN,miyato2018spectral,jolicoeur2019connections}, though it remains an issue to this day. DSM-ALS, on the other hand, does not suffer from that problem since ALS allows for sampling from the full distribution captured by the score network. Nevertheless, the perceptual quality of DSM-ALS higher-resolution images has so far been inferior to that of GAN-generated images. Generative modeling has since seen some incredible work from \citet{ho2020denoising}, who achieved exceptionally low (better) FID on image generation tasks. Their approach showcased a diffusion-based method \citep{sohl2015deep, goyal2017variational} that shares close ties with DSM-ALS, and additionally proposed a convincing network architecture derived from \citet{salimans2017pixelcnn++}.

In this paper, after introducing the necessary technical background in the next section, we build upon the work of \citet{song2020improved} and propose improvements based on theoretical analyses both at training and sampling time. Our contributions are as follows:

\begin{itemize}
 \item We propose Consistent Annealed Sampling (CAS) as a more stable alternative to ALS, correcting inconsistencies relating to the scaling of the added noise;
 \item We show how to recover the \em expected denoised sample \em (EDS) and demonstrate its unequivocal benefits \em w.r.t \em the FID. Notably, we show how to resolve the mismatch observed in DSM-ALS between the visual quality of generated images and its high (worse) FID;
 \item We propose to further exploit the EDS through a hybrid objective function, combining GAN and Denoising Score Matching objectives, thereby encouraging the EDS of the score network to be as realistic as possible.
\end{itemize}

In addition, we show that the network architecture used used by \citet{ho2020denoising} significantly improves sample quality over the RefineNet \citep{lin2017refinenet} architecture used by \citet{song2020improved}. In an ablation study performed on CIFAR-10 and LSUN-church, we demonstrate how these contributions bring DSM-ALS in range of the state-of-the-art for image generation tasks \em w.r.t. \em the FID. The code to replicate our experiments is publicly available at 
\textbf{[Available in supplementary material]}.
%https://github.com/AlexiaJM/AdversarialConsistentScoreMatching.

\section{Background}\label{sec:back}

\subsection{Denoising Score Matching}

Denoising Score Matching (DSM) \citep{hyvarinen2005estimation} consists of training a score network to approximate the gradient of the log density of a certain distribution ($\nabla_{\boldsymbol{x}} \log p(\boldsymbol{x})$), referred to as the score function. This is achieved by training the network to approximate a noisy surrogate of $p$ at multiple levels of Gaussian noise corruption \citep{vincent2011connection}.
The score network $s$, parametrized by $\theta$ and conditioned on the noise level $\sigma$, is tasked to minimize the following loss:
\begin{equation}\label{eqn:score}
\frac{1}{2} ~ \mathbb{E}_{p(\boldsymbol{\tilde{x}}, 
\boldsymbol{x}, \sigma)}\left[ \norm{\sigma s_{\theta}(\boldsymbol{\tilde{x}},\sigma) + \frac{\boldsymbol{\tilde{x}} - \boldsymbol{x}}{\sigma} }_2^2 \right],
\end{equation}

where $p(\boldsymbol{\tilde{x}}, \boldsymbol{x}, \sigma) = q_\sigma(\boldsymbol{\tilde{x}} | \boldsymbol{x})p(\boldsymbol{x})p(\sigma)$. We define further  $q_{\sigma}(\boldsymbol{\tilde{x}}|\boldsymbol{x}) = \mathcal{N}(\boldsymbol{\tilde{x}} | {\boldsymbol{x}},\sigma^2 \boldsymbol{I})$ the corrupted data distribution, $p(\boldsymbol{x})$ the training data distribution, and $p(\sigma)$ the uniform distribution over a set $\{\sigma_i\}$ corresponding to different levels of noise.
In practice, this set is defined as a geometric progression between $\sigma_1$ and $\sigma_L$ (with $L$ chosen according to some computational budget):
\begin{align}\label{sigma_set}\left\{\sigma_i\right\}_{i=1}^L = \left\{\gamma^i\sigma_1 ~ \Big| \hspace{4pt} i \in \{0, \dots, L-1\},\gamma \triangleq \frac{\sigma_2}{\sigma_1} = ... = \left( \frac{\sigma_L}{\sigma_1} \right)^{\frac{1}{L-1}} < 1\right\}.
\end{align}

Rather than having to learn a different score function for every $\sigma_i$, one can train an unconditional score network by defining $s_{\theta}(\boldsymbol{\tilde{x}},\sigma_i)=s_{\theta}(\boldsymbol{\tilde{x}})/\sigma_i$, and then minimizing Eq. \ref{eqn:score}. While unconditional networks are less heavy computationally, it remains an open question whether conditioning helps performance. \citet{li2019learning} and \citet{song2020improved} found that the unconditional network produced better samples, while \citet{ho2020denoising} obtained better results than both of them using a conditional network. Additionally, the denoising autoencoder described in \citet{AR-DAE} gives evidence supporting the benefits of conditioning when the noise becomes small (also see App.~\ref{app:C} and~\ref{app:D} for a theoretical discussion of the difference).
While our experiments are conducted with unconditional networks, we believe our techniques can be straightforwardly applied to conditional networks; we leave that extension for future work.

\subsection{Annealed Langevin Sampling}

Given a score function, one can use Langevin dynamics (or Langevin sampling) \citep{welling2011bayesian} to sample from the corresponding probability distribution. In practice, the score function is generally unknown and estimated through a score network trained to minimize Eq. \ref{eqn:score}. \citet{song2019generative} showed that Langevin sampling has trouble exploring the full support of the distribution when the modes are too far apart and proposed Annealed Langevin Sampling (ALS) as a solution. ALS starts sampling with a large noise level and progressively anneals it down to a value close to $0$, ensuring both proper mode coverage and convergence to the data distribution. Its precise description is shown in Algorithm \ref{alg:anneal1}. 

\begin{minipage}{0.49\textwidth}
\begin{algorithm}[H]
	\caption{Annealed Langevin Sampling}
	\label{alg:anneal1}
	\begin{algorithmic}[1]
	    \Require{$s_\theta, \{\sigma_i\}_{i=1}^L, \epsilon, n_\sigma$.}
	    \vspace{0.06cm}
	    \State{Initialize $\boldsymbol{x}$}
	    \vspace{0.12cm}
	    \For{$i \gets 1$ to $L$}
	        %\State{$\boldsymbol{x}_i \gets \boldsymbol{x}_{i-1}$}
	        \State{$\alpha_i \gets \epsilon ~ \sigma_i^2/\sigma_L^2$}
	        \For{$t \gets 1$ to $n_\sigma$}
                \State{Draw $\boldsymbol{z} \sim \mathcal{N}(0, \boldsymbol{I})$}
                \State{$\boldsymbol{x} \gets \boldsymbol{x} + \alpha_i s_\theta(\boldsymbol{x}, \sigma_i) + \sqrt{2 \alpha_i} \boldsymbol{z}$}
            \EndFor
        \EndFor
        \item[]
        \Return{$\boldsymbol{x}$}
        
	\end{algorithmic}
\end{algorithm}
\end{minipage}
\hfill
\begin{minipage}{0.50\textwidth}
\begin{algorithm}[H]
	\caption{ Consistent Annealed Sampling}
	\label{alg:anneal2}
	\begin{algorithmic}[1]
	    \Require{$s_\theta, \{\sigma_i\}_{i=1}^{L}, \gamma , \epsilon$, $\sigma_{L+1} = 0$} %\Comment{$L = (L-1) n_{\sigma} + 1$} 
	    \State{Initialize $\boldsymbol{x}$}
	   % \State{$\beta \gets \sqrt{1 - \left(\frac{1 - \epsilon/\sigma_{L}^2}{\gamma}\right)^2}$}
        \State{$\beta \gets \sqrt{1 - \left(1 - \epsilon/\sigma_{L}^2\right)^2 / \gamma^2}$}
        \For{$i \gets 1$ to $L$}
	        \State{$\alpha_i \gets \epsilon ~ \sigma_i^2/\sigma_L^2$}
	        %\State{$\sigma_{i+1} \gets \gamma\sigma_{i}$}
            \State{Draw $\boldsymbol{z} \sim \mathcal{N}(0, \boldsymbol{I})$} 
           \State{$\boldsymbol{x} \gets \boldsymbol{x} + \alpha_i s_\theta(\boldsymbol{x}, \sigma_{i}) + \beta \sigma_{i+1}  \boldsymbol{z}$}
        \EndFor
        \item[]
        \Return{$\boldsymbol{x}$}
	\end{algorithmic}
\end{algorithm}
\end{minipage}

\subsection{Expected denoised sample (EDS)}\label{sec:EDS}

A little known fact from Bayesian literature is that one can recover a denoised sample from the score function using the Empirical Bayes mean \citep{robbins1955empirical,miyasawa1961empirical, raphan2011least}:
\begin{equation}\label{eqn:eds}s^{*}(\boldsymbol{\tilde{x}},\sigma) = \frac{H^*(\boldsymbol{\tilde{x}},\sigma) - \boldsymbol{\tilde{x}}}{\sigma^2},\end{equation}
where $H^*(\boldsymbol{\tilde{x}},\sigma) \triangleq \mathbb{E}_{\boldsymbol{x} \sim q_\sigma(\boldsymbol{x} | \boldsymbol{\tilde{x}})}[\boldsymbol{x}]$ is the expected denoised sample given a noisy sample (or Empirical Bayes mean), conditioned on the noise level. A different way of reaching the same result is through the closed-form of the optimal score function, as presented in Appendix \ref{app:C}. The corresponding result for unconditional score function is presented in Appendix \ref{app:D} for completeness.

The EDS corresponds to the expected real image given a corrupted image; it can be thought of as what the score network believes to be the true image concealed within the noisy input. It has also been suggested that denoising the samples (i.e., taking the EDS) at the end of the Langevin sampling improves their quality \citep{saremi2019neural, li2019learning, kadkhodaie2020solving}. In Section \ref{sec:sampling_improvements}, we provide further evidence that denoising the final Langevin sample brings it closer to the assumed data manifold. In particular, we show that the Fréchet Inception Distance (FID) consistently decreases (improves) after denoising. Finally, in Section \ref{sec:adversarial}, we build a hybrid training objective using the properties of the EDS discussed above.

\section{Consistent scaling of the noise}\label{sec:consistent}

%In this section, we present Consistent Annealed Sampling (CAS) as an alternative to ALS, demonstrate how it can enforce any sampling noise schedule, and explain how this improves on ALS.
In this section, we present inconsistencies in ALS relating to the noise scaling and introduce Consistent Annealed Sampling (CAS) as an alternative.

\subsection{Inconsistencies in ALS}

One can think of the ALS algorithm as a sequential series of Langevin Dynamics (inner loop in Algorithm \ref{alg:anneal1}) for decreasing levels of noise (outer loop). If allowed an infinite number of steps $n_\sigma$, the sampling process will properly produce samples from the data distribution. 

In ALS, the score network is conditioned on geometrically decreasing noise ($\sigma_i$). In the unconditional case, this corresponds to dividing the score network by the noise level (\em i.e., \em $s_{\theta}(\boldsymbol{\tilde{x}},\sigma_i)=s_{\theta}(\boldsymbol{\tilde{x}})/\sigma_i$). Thus, in both conditional and unconditional cases, we make the assumption that the noise of the sample at step $i$ will be of variance $\sigma_i^2$, an assumption upon which the quality of the estimation of the score depends. 
While choosing a geometric progression of noise levels seems like a reasonable (though arbitrary) schedule to follow, we show that ALS does not ensure such schedule.

Assume we have the true score function $s^*$ and begin sampling using a real image with some added zero-centered Gaussian noise of standard deviation $\sigma_0 = 50$. In Figure \ref{fig_std}, we illustrate how the intensity of the noise in the sample evolves through ALS and CAS, our proposed sampling, for a given sampling step size $\epsilon$ and a geometric schedule  in this idealized scenario. We note that, although a large $n_\sigma$ approaches the real geometric curve, it will only reach it at the limit ($n_\sigma \to \infty$ and $\epsilon \to 0$). 
%Furthermore, if $\epsilon$ is inappropriately chosen for $n_\sigma$ (as it is at $n_\sigma=1$), the noise levels can never converge to $\sigma_L$, even when using the best possible score function. 
Most importantly, Figure \ref{fig_diff_std} highlights how even when the annealing process does converge, the progression of the noise is never truly geometric; we prove this formally in Proposition \ref{prop:1}.

\begin{figure}[ht!] 
\captionsetup[subfigure]{justification=centering}
    \centering
    \begin{subfigure}[t]{0.465\linewidth}
\includegraphics[width=\linewidth]{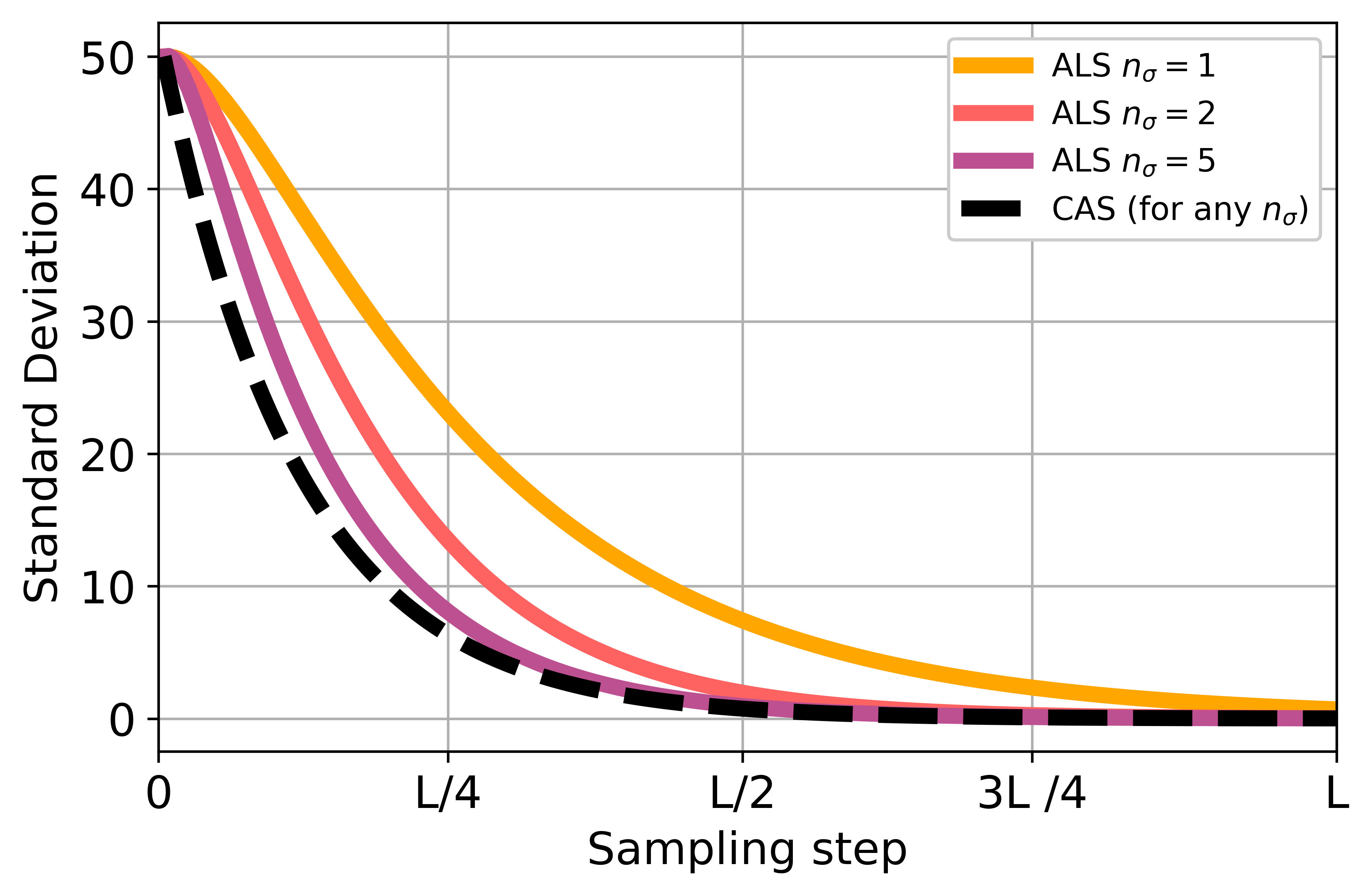}
    \caption{Standard deviation of the noise in the image} 
    \label{fig_std} 
  \end{subfigure}
  \begin{subfigure}[t]{0.47\linewidth}
    \includegraphics[width=\linewidth]{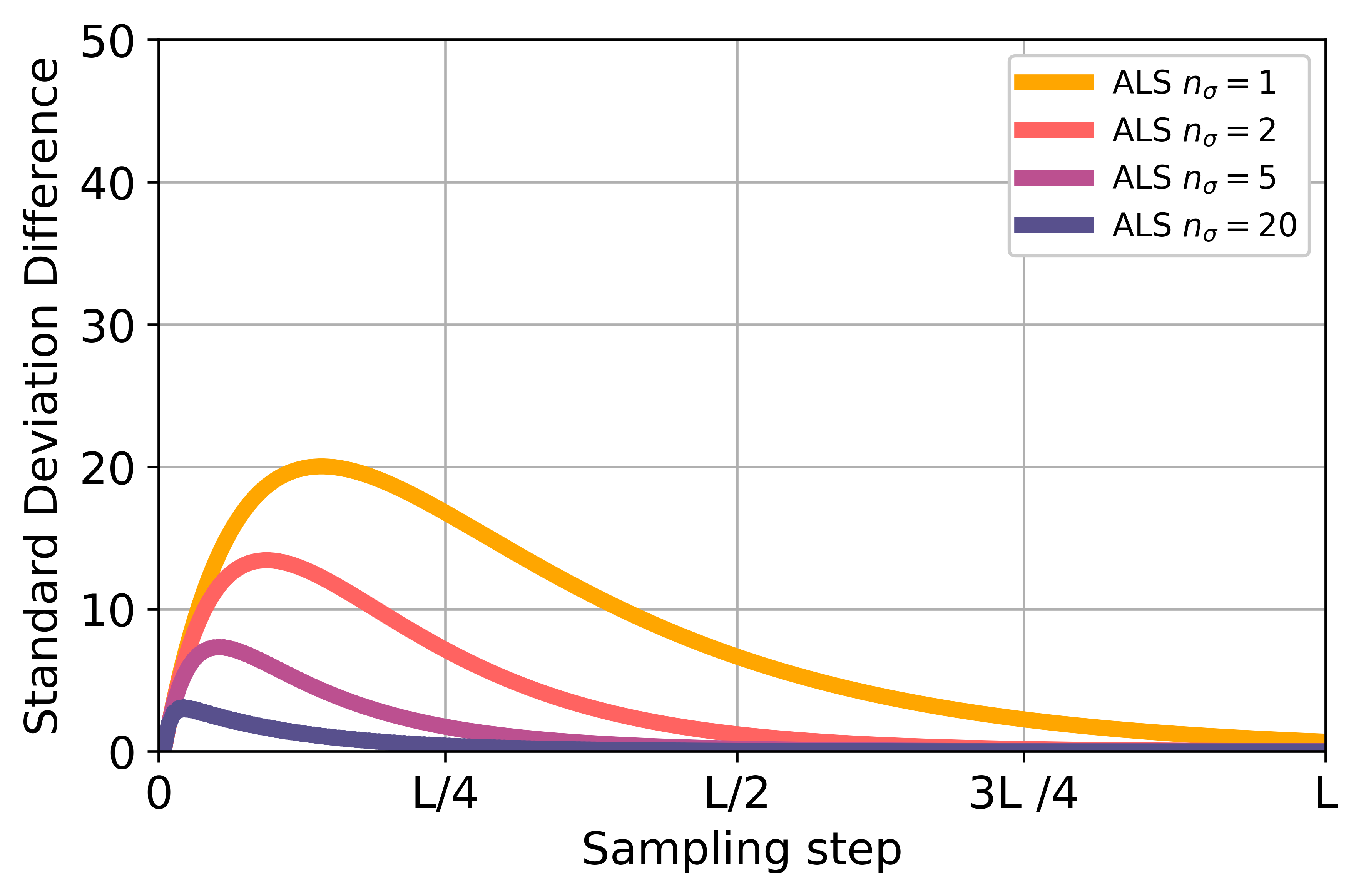}
    \caption{Difference between the standard deviation from ALS and CAS } 
    \label{fig_diff_std} 
  \end{subfigure} 
  \caption{Standard deviation during idealized sampling using a perfect score function $s^*$. The black curve in (a) corresponds to the true geometric progression, as demonstrated in Proposition \ref{proof:2}.}
  \label{fig:std} 
\end{figure}

\begin{proposition}\label{prop:1}
\em 
Let $s^*$ be the optimal score function from Eq. \ref{eqn:eds}. Following the sampling described in Algorithm \ref{alg:anneal1}, the variance of the noise component in the sample $\boldsymbol{x}$ will remain greater than $\sigma_t^2$ at every step $t$. \em 
\end{proposition}

The proof is presented in Appendix \ref{app:als_bad_proof}.
In particular, for $n_\sigma < \infty$, sampling has not fully converged and the remaining noise is carried over to the next iteration of Langevin Sampling. It also follows that for any $s_\theta$ different from the optimal $s^*$, the actual noise at every iteration is expected to be even higher than for the best possible score function $s^*$. 

\subsection{Algorithm}

We propose Consistent Annealed Sampling (CAS) as a sampling method that ensures the noise level will follow a prescribed schedule for any sampling step size $\epsilon$ and number of steps $L$. Algorithm \ref{alg:anneal2} illustrates the process for a geometric schedule. Note that for a different schedule, $\beta$ will instead depend on the step $t$, as in the general case, $\gamma_t$ is defined as $\sigma_{t+1} / \sigma_{t}$.

\begin{proposition}\label{proof:2}
\em Let $s^*$ be the optimal score function from Eq. \ref{eqn:eds}. Following the sampling described in Algorithm \ref{alg:anneal2}, the variance of the noise component in the sample $\boldsymbol{x}$ will consistently be equal to $\sigma_t^2$ at every step $t$. \em 
\end{proposition}

The proof is presented in Appendix \ref{app:cas_good_proof}. Importantly, Proposition \ref{proof:2} holds no matter how many steps $L$ we take to decrease the noise geometrically. For ALS, $n_{\sigma}$ corresponds to the number of steps per level of noise. It plays a similar role in CAS: we simply dilate the geometric series of noise levels used during training by a factor of $n_\sigma$, such that $L_\text{sampling} = (L_\text{training} - 1) n_\sigma + 1$. 
Note that the proposition only holds when the initial sample is a corrupted image ($i.e., ~~ \boldsymbol{x}_0 = \mathcal{I} + \sigma_0 \boldsymbol{z}_0$). However, by defining $\sigma_0$ as the maximum Euclidean distance between all pairs of training data points \citep{song2020improved},
the noise becomes in practice much greater than the true image; sampling with pure noise initialization ($i.e., ~~ \boldsymbol{x}_0 = \sigma_0\boldsymbol{z}_t$) becomes indistinguishable from sampling with data initialization.

\section{Benefits of the EDS on synthetic data and image generation }\label{sec:sampling_improvements}

As previously mentioned, it has been suggested that one can obtain better samples (closer to the assumed data manifold) by taking the EDS of the last Langevin sample. We provide further evidence of this with synthetic data and standard image datasets. 

It can first be observed that the sampling steps correspond to an interpolation between the previous point and the EDS, followed by the addition of noise.

\begin{proposition}\label{prop:interp}
\em Given a noise-conditional score function, the update rules from Algorithm \ref{alg:anneal1} and Algorithm \ref{alg:anneal2} are respectively equivalent to the following update rules: \em 
\begin{align*} &\boldsymbol{x} \gets (1-\eta)\boldsymbol{x} + \eta H(\boldsymbol{x},\sigma_i) + \sqrt{2\eta}\sigma_i \boldsymbol{z}  && \text{for } \boldsymbol{z} \sim \mathcal{N}(0, \boldsymbol{I}) ~~ \text{and} ~~ \eta = \frac{\epsilon}{\sigma_L^2} \\
&\boldsymbol{x} \gets (1-\eta)\boldsymbol{x} + \eta H(\boldsymbol{x},\sigma_i) + \beta \sigma_{i+1} \boldsymbol{z}
\end{align*}
\end{proposition}

The demonstration is in Appendix \ref{app:update_rule_proof}. This result is equally true for an unconditional score network, with the distinction that $\eta$ would no longer be independent of $\sigma_i$ but rather linearly proportional to it.

Intuitively, this implies that the sampling steps slowly move the current sample towards a moving target (the EDS). If the sampling behaves appropriately, we expect the final sample $\boldsymbol{x}$ to be very close to the EDS, \em i.e.,\em ~  $\boldsymbol{x} \approx H(\boldsymbol{x}, \sigma_L)$. However, if the sampling step size is inappropriate, or if the EDS does not stabilize to a fixed point near the end of the sampling, these two quantities may be arbitrarily far from one another. As we will show, the FIDs from \citet{song2020improved} suffer from such  distance. 

The equivalence showed in Proposition \ref{prop:interp} suggests instead to take the expected denoised sample at the end of the Langevin sampling as the final sample; this would be equivalent to the update rule
$ \boldsymbol{x} \gets H(\boldsymbol{x}, \sigma_L)$ at the last step. Synthetic 2D examples shown in Figure \ref{fig1} demonstrate the immediate benefits of this technique.

\begin{figure}[ht] 
  \begin{subfigure}[b]{0.5\linewidth}
    \centering
    \includegraphics[width=0.8\linewidth]{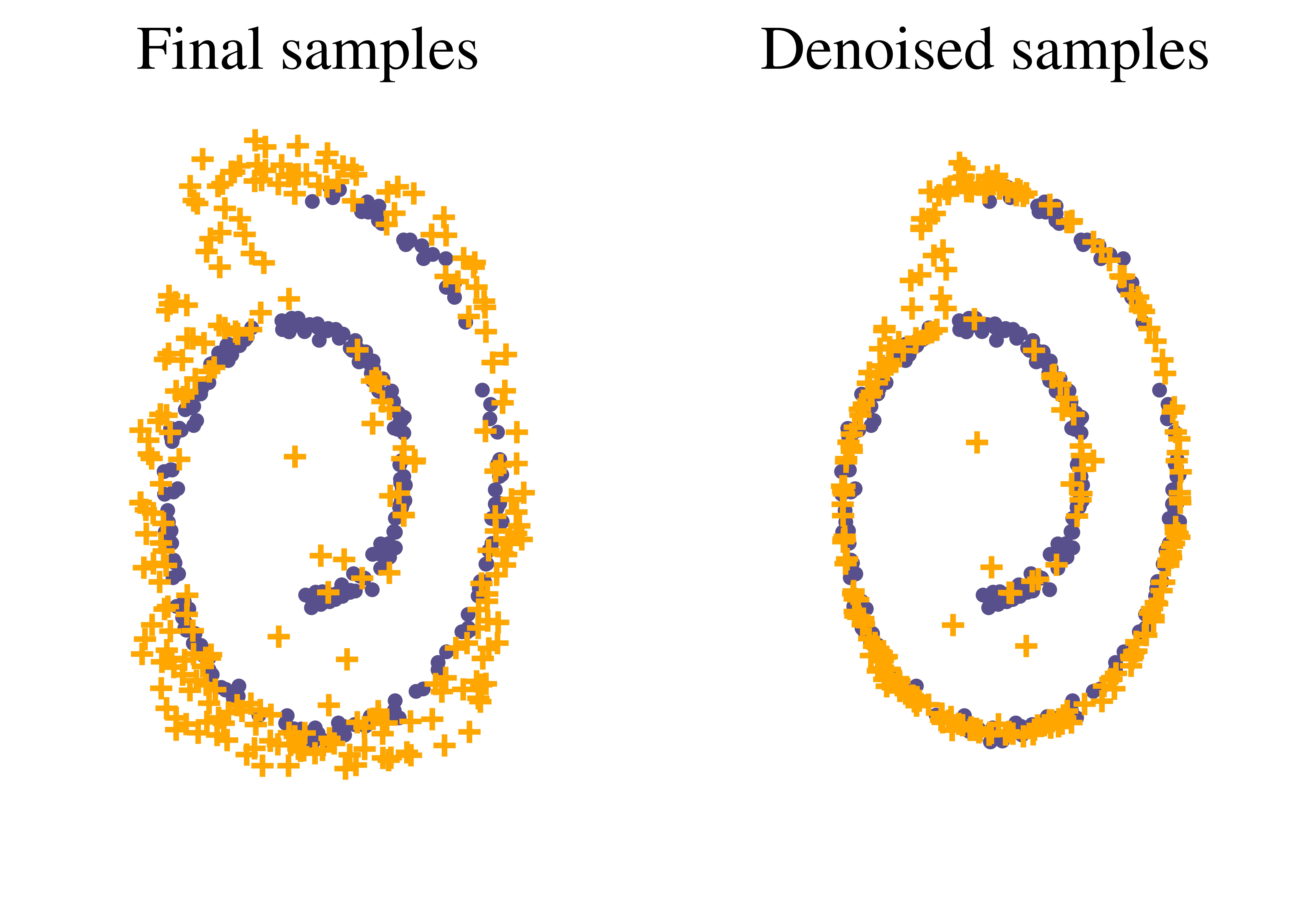} 
    \vspace{-0.5cm}
    \caption{Swiss Roll dataset} 
  \end{subfigure}
  \begin{subfigure}[b]{0.5\linewidth}
    \centering
    \includegraphics[width=0.8\linewidth]{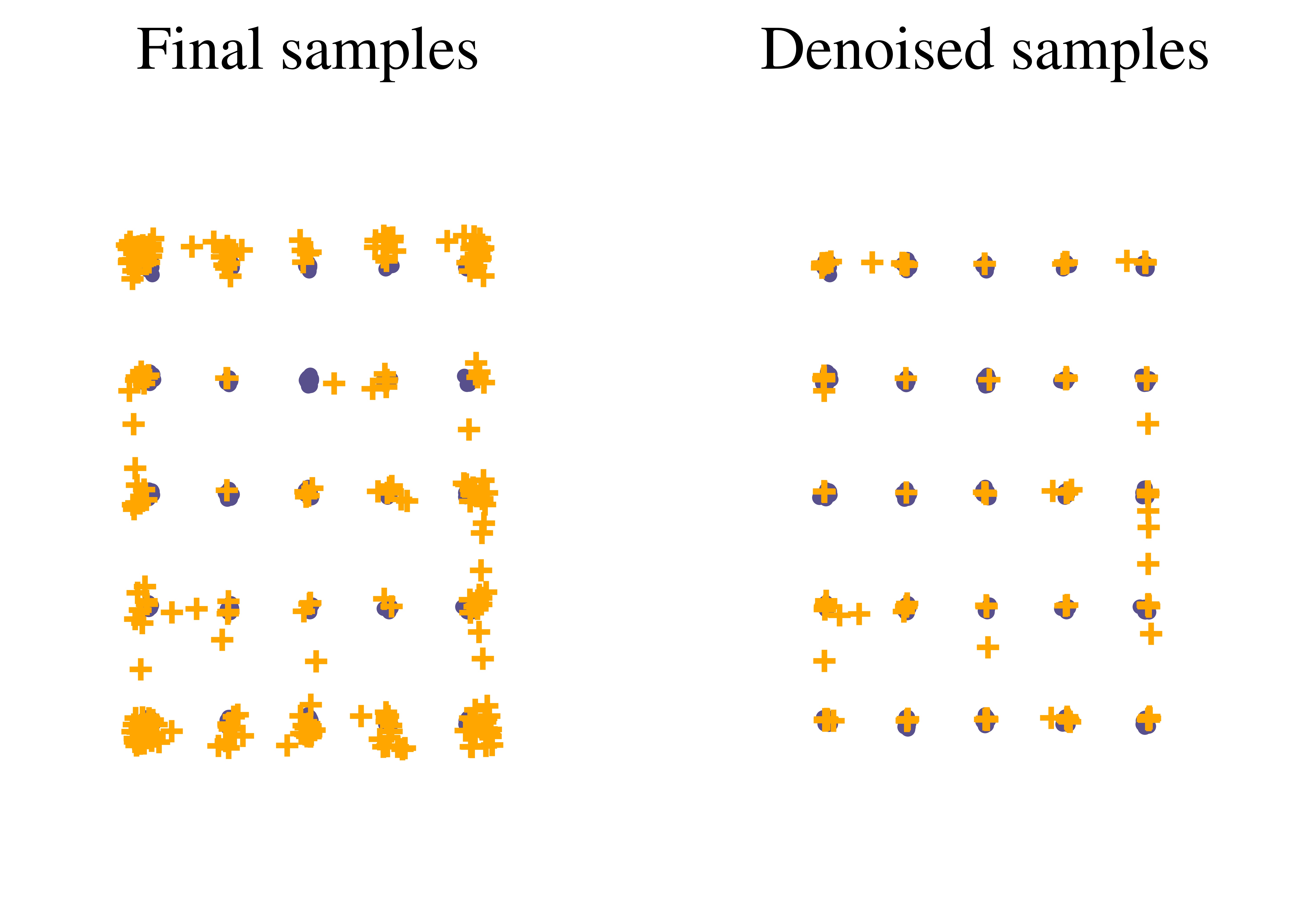}
    \vspace{-0.5cm}
    \caption{25 Gaussians dataset} 
  \end{subfigure}
  \caption{Langevin sampling on synthetic 2D experiments. Circles are real data points, crosses are generated data points. On both datasets, taking the EDS brings the samples much closer to the real data manifold.}
  \label{fig1} 
\end{figure}

\begin{figure}[h] 
    \centering
    \includegraphics[width=0.5\linewidth]{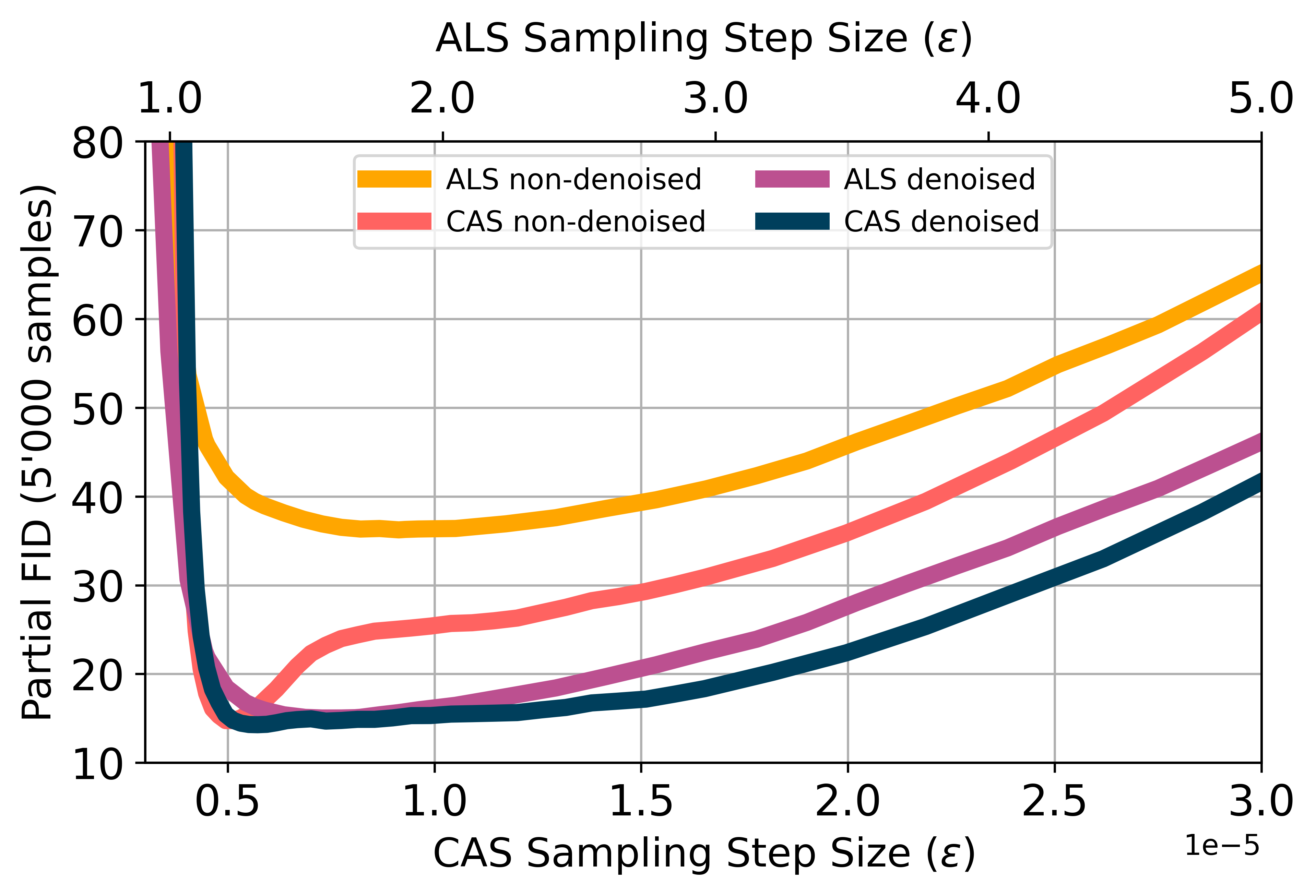} 
  
  %\caption{Partial estimate of FID (lower is better) as a function of the sampling step size on CIFAR-10 at $n_\sigma = 1$ (left) for CAS and (right) for ALS. Displayed are the curves for non-denoised and denoised samples when sampling with the same \{$\sigma_i$\} as in training.}
  \caption{Partial estimate of FID (lower is better) as a function of the sampling step size on CIFAR-10, with $n_\sigma = 1$. The interactions between \em consistent sampling \em and denoising are shown.
  }
  \label{lr_rate_cifar10} 
\end{figure}

We train a score network on CIFAR-10 \citep{krizhevsky2009learning} and report the FID from both ALS and CAS as a function of the sampling step size and of denoising in Figure \ref{lr_rate_cifar10}. The first observation to be made is just how critical denoising is to the FID score for ALS, even as its effect cannot be perceived by the human eye. For CAS, we note that the score remains small for a much wider range of sampling step sizes when denoising. Alternatively, the sampling step size must be very carefully tuned to obtain results close to the optimal. 
%Thus, using the EDS leads to significantly more stability with respect to the FID for CAS.

Figure \ref{lr_rate_cifar10} also shows that, with CAS, the FID of the final sample is approximately equal to the FID of the denoised samples for small sampling step sizes. Furthermore, we see a smaller gap in FID between denoised and non-denoised for larger sampling step sizes than ALS. This suggests that consistent sampling is resulting in the final sample being closer to the assumed data manifold (i.e., $\boldsymbol{x} \approx H_\theta(\boldsymbol{x}, \sigma_L)$).

Interestingly, when \citet{song2020improved} improved their score matching method, they could not explain why the FID of their new model did not improve even though the generated images looked better visually. To resolve that matter, they proposed the use of a new metric \citep{zhou2019hype} that did not have this issue. As shown in Figure \ref{lr_rate_cifar10}, denoising resolves this mismatch.

\section{Adversarial formulation}\label{sec:adversarial}

The score network is trained to recover an uncorrupted image from a noisy input minimizing the $l_2$ distance between the two. However, it is well known from the image restoration literature that $l_2$ does not correlate well with human perception of image quality \citep{zhang2012comprehensive,zhao2016loss}.
One way to take advantage of the EDS would be to encourage the score network to produce an EDS that is more realistic from the perspective of a discriminator. 
Intuitively, this would incentivize the score network to produce more discernible features at inference time.

We propose to do so by training the score network to simultaneously minimize the score-matching loss function and maximize the probability of denoised samples being perceived as real by a discriminator. We use alternating gradient descent to sequentially train a discriminator for a determined number of steps at every score function update. 

In our experiments, we selected the Least Squares GAN (LSGAN) \citep{LSGAN} formulation as it performed best (see Appendix \ref{app:exp_details} for details). For an unconditional score network, the objective functions are as follows:

\begin{align}\label{eqn:GANmax}
&\max_{\phi} ~ \mathbb{E}_{p(\boldsymbol{x})}\left[ (D_\phi(\boldsymbol{x}) - 1)^2 \right] + \mathbb{E}_{p(\boldsymbol{\tilde{x}}, \boldsymbol{x}, \sigma)} \left[ (D_\phi(H_{\theta}(\boldsymbol{\tilde{x}},\sigma)+1)^2\right]
\end{align}
\begin{align}\label{eqn:GANmin}
&\min_{\theta} ~ \mathbb{E}_{p(\boldsymbol{\tilde{x}}, 
\boldsymbol{x}, \sigma)}\left[ (D_\phi(H_{\theta}(\boldsymbol{\tilde{x}},\sigma))-1)^2 + \frac{\lambda}{2} \norm{\sigma s_{\theta}(\boldsymbol{\tilde{x}},\sigma) + \frac{\boldsymbol{\tilde{x}} - \boldsymbol{x}}{\sigma}}_2^2 \right],
\end{align}

where $H_{\theta}(\boldsymbol{\tilde{x}},\sigma) = s_{\theta}(\boldsymbol{\tilde{x}},\sigma) \sigma^2 + \boldsymbol{\tilde{x}}$ is the EDS derived from the score network.
Eq. \ref{eqn:GANmax} is the objective function of the LSGAN discriminator, while Eq. \ref{eqn:GANmin} is the adversarial objective function of the score network derived from Eq. \ref{eqn:score} and from the LSGAN objective function.

We note the similarities between these objective functions and those of an LSGAN adversarial autoencoder \citep{makhzani2015adversarial, tolstikhin2017wasserstein, tran2018dist}, with the distinction of using a denoising autoencoder $H$ as opposed to a standard autoencoder. 

As GANs favor quality over diversity, there is a concern that this hybrid objective function might decrease the diversity of samples produced by the ALS. In Section \ref{sec:quality}, we first study image generation improvements brought by this method and then address the diversity concerns with experiments on the 3-StackedMNIST \citep{metz2016unrolled} dataset in Section $\ref{sec:diversity}$.

\section{Experiments}\label{sec:experiments}

\subsection{Ablation Study}\label{sec:quality}

We ran experiments on CIFAR-10 
%(50k 32x32 images)
\citep{krizhevsky2009learning} and LSUN-churches 
%(126k 64x64 images)
\citep{yu2015lsun} with the score network architecture used by \citet{song2020improved}. We also ran similar experiments with an unconditional version of the network architecture by \citet{ho2020denoising}, given that their approach is similar to \citet{song2019generative} and they obtain very small FIDs.
For the hybrid adversarial score matching approach, we used an unconditional BigGAN discriminator \citep{brock2018large}.
We compared three factors in an ablation study: adversarial training, Consistent Annealed Sampling and denoising.

%rtachet: We have a bunch of DGX-2. In the future, if we want to scale things up we could make a request for one of those.
Details on how the experiments were conducted are found in Appendix \ref{app:exp_details}. Unsuccessful experiments with large images are also discussed in Appendix \ref{app:B}. See also Appendix \ref{sec:is} for a discussion pertaining to the use of the Inception Score \citep{heusel2017gans}, a popular metric for generative models.

Results for CIFAR-10 and LSUN-churches with \citet{song2019generative} score network architecture are respectively shown in Table \ref{tab:cifar10} and \ref{tab:lsun-churches}. Results for CIFAR-10 with \citet{ho2020denoising} score network architecture are shown in Table \ref{tab:cifar102}.

\begin{table}[!ht]
	\caption{[Non-denoised / Denoised FID] from 10k samples on CIFAR-10 (32x32) with \citet{song2019generative} score network architecture}
	\label{tab:cifar10}
	\centering
	\begin{tabular}{ccc}
		\toprule
		Sampling & Non-adversarial & Adversarial \\
		\cmidrule(){1-3}
	    non-consistent $(n_{\sigma} = 1)$ & 36.3 / 13.3 & 30.0 / 11.8 \\
	    non-consistent $(n_{\sigma} = 5)$ & 33.7 / 10.9 & 26.4 / {\fontseries{b}\selectfont 9.5} \\
	    \cmidrule(){1-3}
		consistent $(n_{\sigma} = 1)$ & 14.7 / 12.3 & 11.9 / 10.8 \\
		consistent $(n_{\sigma} = 5)$ & 12.7 / 11.2 & 9.9 / 9.7 \\
		%consistent $(n_{\sigma} = 15)$ & 12.4 / {\fontseries{b}\selectfont 10.8} & 11.5 / 10.3 \\
		\bottomrule
	\end{tabular}
\end{table}
\begin{table}[!ht]
	\caption{[Non-denoised / Denoised FID] from 10k samples on LSUN-Churches (64x64) with \citet{song2019generative} score network architecture}
	\label{tab:lsun-churches}
	\centering
	\begin{tabular}{ccc}
		\toprule
		Sampling & Non-adversarial & Adversarial \\
		\cmidrule(){1-3}
	    non-consistent $(n_{\sigma} = 1)$ & 43.2 / 40.3 & 40.9 / 36.7  \\
	    non-consistent $(n_{\sigma} = 5)$ & 42.0 / 39.2 & 40.0 / 35.8 \\
	    \cmidrule(){1-3}
	    consistent $(n_{\sigma} = 1)$ & 41.5 / 40.7 & 38.2 / 36.7 \\
		consistent $(n_{\sigma} = 5)$ & 39.5 / 39.1 & 36.3 / {\fontseries{b}\selectfont 35.4} \\
		\bottomrule
	\end{tabular}
\end{table}
\begin{table}[!ht]
	\caption{[Non-denoised / Denoised FID] from 10k samples on CIFAR-10 (32x32) with \citet{ho2020denoising} unconditional score network architecture}
	\label{tab:cifar102}
	\centering
	\begin{tabular}{ccc}
		\toprule
		Sampling & Non-adversarial & Adversarial \\
		\cmidrule(){1-3}
	    non-consistent $(n_{\sigma} = 1)$ & 25.3 / 7.5 & 21.6 / 7.5 \\
	    non-consistent $(n_{\sigma} = 5)$ & 20.0 / {\fontseries{b}\selectfont 5.6} & 17.7 / 6.1 \\
	    \cmidrule(){1-3}
		consistent $(n_{\sigma} = 1)$ & 7.8 / 7.1 & 7.7 / 7.1 \\
		consistent $(n_{\sigma} = 5)$ & 6.2 / 6.1 & 6.1 / 6.5 \\
		\bottomrule
	\end{tabular}
\end{table}

We always observe an improvement in FID from denoising and by increasing $n_\sigma$ from 1 to 5. 
We observe an improvement from using the adversarial approach with \citet{song2019generative} network architecture, but not on denoised samples with the \citet{ho2020denoising} network architecture. We hypothesize that this is a limitation of the architecture of the discriminator since, as far as we know, no variant of BigGAN achieves an FID smaller than 6. Nevertheless, it remains advantageous for more simple architectures, as shown in Table \ref{tab:cifar10} and \ref{tab:lsun-churches}. We observe that consistent sampling outperforms non-consistent sampling on the CIFAR-10 task at $n_\sigma = 1$, the quickest way to sample. 
%The otherwise slight performance deficit of the consistent sampling gives evidence that geometrically decreasing the noise, though intuitively justified, might not always be the optimal choice \footnote{The schedule from \cite{ho2020denoising} can actually be reformulated as an almost linear schedule.}.

We calculated the FID of the non-consistent denoised models from 50k samples in order to compare our method with the recent work from \citet{ho2020denoising}. We obtained a score of 3.65 for the non-adversarial method and 4.02 for the adversarial method on the CIFAR-10 task when sharing their architecture; these scores are close to their reported 3.17. Although not explicit in their approach, \citet{ho2020denoising} denoised their final sample. This suggests that taking the EDS and using an architecture akin to theirs were the two main reasons for outperforming \citet{song2020improved}.  Of note, our method only trains the score network for 300k iterations, while \citet{ho2020denoising} trained their networks for more than 1 million iterations to achieve similar results.

\subsection{Non-adversarial and Adversarial score networks have equally high diversity}\label{sec:diversity}

To assess the diversity of generated samples, we evaluate our models on the 3-Stacked MNIST generation task \citep{metz2016unrolled} (128k images of 28x28), consisting of numbers from the MNIST dataset \citep{lecun1998gradient} superimposed on 3 different channels. We trained non-adversarial and adversarial score networks in the same way as the other models. The results are shown in Table \ref{fig:stacked}. 

We see that each of the 1000 modes is covered, though the KL divergence is still inferior to PACGAN \citep{lin2018pacgan}, meaning that the mode proportions are not perfectly uniform. Blindness to mode proportions is thought to be a fundamental limitation of score-based methods \citep{wenliang2020blindness}. Nevertheless, these results confirm a full mode coverage on a task where most GANs struggle and, most importantly, that using a hybrid objective does not hurt the diversity of the generated samples.

\begin{table}[ht]
	\begin{center}
  	\begin{tabular}{ l  r  r    }
  		\hline
		\multicolumn{2}{c}{3-Stacked MNIST}  \\ \cline{2-3} %&CIFAR-10 \\ \cline{2-3}
    		 & Modes (Max 1000)& KL \\ \hline %& IvOM  \\ \hline
    		DCGAN \citep{DCGAN} 	& 99.0 & 3.40 		\\%& 0.00844$\pm$0.0020 \\ 
    		ALI  \citep{dumoulin2016adversarially} & 16.0 &  5.40 		\\%& 0.00670$\pm$0.0040 \\
    		Unrolled GAN \citep{metz2016unrolled} & 48.7 & 4.32 		\\%&  0.01300$\pm$0.0009\\
    		VEEGAN \citep{srivastava2017veegan} 		& 150.0 &  2.95		\\%&  0.00680$\pm$0.0001\\\hline
		    PacDCGAN2 \citep{pacgan} 	& 1000.0 & 0.06  \\
		    WGAN-GP \citep{kumar2019maximum,WGAN-GP} & 959.0 & 0.73 \\ 
		    PresGAN \citep{dieng2019prescribed} & 999.6 & 0.115 \\
		    MEG \citep{kumar2019maximum} & 1000.0 & 0.03 \\
		    \hline
		    Non-adversarial DSM (ours) 	& 1000.0 & 1.36 \\
		    Adversarial DSM (ours) 	& 1000.0 & 1.49 \\
    		\hline
  	\end{tabular}
	\end{center}
	\caption{As in \citet{lin2018pacgan}, we generated 26k samples and evaluated the mode coverage and KL divergence based on the predicted modes from a pre-trained MNIST classifier.
	}
	\label{fig:stacked}
\end{table}

\section{Conclusion}

We proposed Consistent Annealed Sampling as an alternative to Annealed Langevin Sampling, which ensures the expected geometric progression of the noise and brings the final samples closer to the data manifold. We showed how to extract the expected denoised sample and how to use it to further improve the final Langevin samples. We proposed a hybrid approach between GAN and score matching. With experiments on synthetic and standard image datasets; we showed that these approaches generally improved the quality/diversity of the generated samples.

We found equal diversity (coverage of all 1000 modes) for the adversarial and non-adversarial variant of the difficult StackedMNIST problem. Since we also observed better performance (from lower FIDs) in our other adversarial models trained on images, we conclude that making score matching adversarial increases the quality of the samples without decreasing diversity. 
%This suggests that making score matching adversarial may be more beneficial than using GANs given the guarantees on diversity (theoretically from Langevin Sampling and experimentally).
These findings imply that score matching performs better than most GANs and on-par with state-of-the-art GANs. Furthermore, our results suggest that hybrid methods, combining multiple generative techniques together, are a very promising direction to pursue.

As future work, these models should be scaled to larger batch sizes on high-resolution images, since GANs have been shown to produce outstanding high-resolution images at very large batch sizes (2048 or more). We also plan to further study the theoretical properties of CAS by considering its corresponding stochastic differential equation. % which underlies CAS in order to gain a better understanding of the underlying generative process.
%We believe analysis of convergence guarantees of CAS would give us a better understanding of the underlying generative process. %Finally, contrastive learning \citep{hadsell2006dimensionality, chen2020simple} is a promising technique that has been shown to improve the quality of images in GANs \citep{zhao2020differentiable}; this technique could be used with adversarial score matching to capture more realistic features with the score network.

\iffalse

\subsubsection*{Acknowledgments}

We  would like to thank Yang Song, Linda Petrini, Florian Bordes,  Vikram Voleti, Amartya Mitra, Isabela Albuquerque, and Reyhane Askari for their useful discussions and feedback. We thank Yang Song for suggesting to use the diffusion network architecture. Alexia would like to thank her wife Emy Gervais for her support.

We would also
like to thank Compute Canada and Calcul Québec
for the GPUs which were used in this work. This
work was also partially supported by  the NSERC ES-D grant (ESD3-546493-2020), the FRQNT new
researcher program (2019-NC-257943), the NSERC Discovery grant (RGPIN-2019-06512), a startup grant by
IVADO, a grant by Microsoft Research and a Canada
CIFAR AI chair.
\fi

%\bibliographystyle{apalike}
\bibliographystyle{unsrtnat}
\bibliography{paper}

\clearpage

\appendix
\section*{Appendices}

\section{Broader Impact}\label{app:broader_impact}

Unfortunately, these improvements in image generation come at a very high computational cost, meaning that the ability to generate high-resolution images is constrained by the availability of large computing resources (TPUs or clusters of 8+ GPUs). This is mainly due to the architectures used in this paper, while adding a discriminator further adds to the training computational load.

\section{Experiments details}
\label{app:exp_details}

\begin{table}[H]
    \caption{Sampling learning rates for non-adversarial ($\epsilon$) and adversarial ($\epsilon_{adv}$) score matching}\label{tab:hyperparameters}
    \begin{center}
    \begin{tabular}{cccc|cc}
    \toprule
    Network architecture & Dataset & consistent & $n_\sigma$ & $\epsilon$ & $\epsilon_{adv}$ \\ 
    \midrule
    \citet{song2020improved} & CIFAR-10 & No & 1 & 1.8e-5 & 1.725e-5\\
    \citet{song2020improved} & CIFAR-10 & No & 5 & 3.6e-6 & 3.7e-6\\
    \citet{song2020improved} & CIFAR-10 & Yes & 1 & 5.6e-6 & 5.55e-6\\
    \citet{song2020improved} & CIFAR-10 & Yes & 5 & 1.1e-6 & 1.05e-6\\
    \hline
    \citet{song2020improved} & LSUN-Churches & No & 1 & 4.85e-6 & 4.85e-6\\
    \citet{song2020improved} & LSUN-Churches & No & 5 & 9.7e-7 & 9.7e-7\\
    \citet{song2020improved} & LSUN-Churches & Yes & 1 & 2.8e-6 & 2.8e-6\\
    \citet{song2020improved} & LSUN-Churches & Yes & 5 & 4.5e-7 & 4.5e-7\\
    \hline
    \citet{ho2020denoising} & CIFAR-10 & No & 1 & 1.6e-5 & 1.66e-05\\
    \citet{ho2020denoising} & CIFAR-10 & No & 5 & 4.0e-6 & 4.25e-6\\
    \citet{ho2020denoising} & CIFAR-10 & Yes & 1 & 5.45e-6 & 5.6e-6\\
    \citet{ho2020denoising} & CIFAR-10 & Yes & 5 & 1.05e-6 & 1e-6\\
    \hline
    \citet{song2020improved} & 3-StackedMNIST & Yes & 1 & 5.0e-6 & 5.0e-6\\
    \bottomrule
    \end{tabular}
    \end{center}
\end{table}

Following the recommendations from \citet{song2020improved}, we chose $\sigma_1=50$ and $L=232$ on CIFAR-10, and $\sigma_1=140$, $L=788$ on LSUN-Churches, both with $\sigma_L = 0.01$. We used a batch size of 128 in all models. We first swept summarily the training checkpoint (saved at every 2.5k iterations), the Exponential Moving Average (EMA) coefficient from \{.999, .9999\}, and then swept over the sampling step size $\epsilon$ with approximately 2 significant number precision. The values reported in Table \ref{tab:cifar10} correspond to the sampling step size that minimized the denoised FID for every $n_\sigma$ (See Table \ref{tab:hyperparameters}). We used the same sampling step sizes for adversarial and non-adversarial. Empirically, the optimal sampling step size is found for a certain $n_\sigma $ and is extrapolated to other precision levels by solving $\beta_{n_\sigma'} = \beta_{n_\sigma} \sqrt{n_\sigma/n_\sigma'}$ for the consistent algorithm. In the non-consistent algorithm, we found the best sampling step size at $n_\sigma = 1$ and divided by 5 to obtain a starting point to find the optimal value at $n_\sigma = 5$. The best EMA values were found to be .9999 in CIFAR-10 and .999 in LSUN-churches. The number of score network training iterations was 300k on CIFAR-10 and 200k on LSUN-Churches.

Of note, \citet{song2020improved} used non-denoised non-consistent sampling with $n_\sigma=5$ and $n_\sigma=4$ for CIFAR-10 and LSUN-Churches respectively. However, they did not use the same learning rates (we tuned ours more precisely) and they used bigger images for LSUN-Churches.

Regarding the adversarial approach, we swept for the GAN loss function, the number of discriminator steps per score network steps ($n_D \in \{1,2\}$), the Adam optimizer \citep{Adam} parameters, and the hyperparameter $\lambda$ (see Eq. \ref{eqn:GANmin}) based on quick experiments on CIFAR-10. LSGAN \citep{LSGAN} yielded the best FID scores among all other loss functions considered, namely the original GAN loss function \citep{GAN}, HingeGAN \citep{lim2017geometric}, as well as their relativistic counterparts \citep{jolicoeur2018relativistic,jolicoeur2019relativistic}. Note that the saturating variant (see \citet{GAN}) on LSGAN worked as well as its non-saturating version; we did not use it for simplicity. Following the trend towards zero or negative momentum \citep{gidel2019negative}, we used the Adam optimizer with hyperparameters $(\beta_1,\beta_2)=(-.5,.9)$ for the discriminator and $(\beta_1,\beta_2)=(0,.9)$ for the score network. These values were found by sweeping over $\beta_1 \in \{.9,.5,0,-.5\}$ and $\beta_2 \in \{.9, .999\}$. We found the simple setting $\lambda=1$ to perform comparatively better than more complex weighting schemes. We used $n_D=2$ on CIFAR-10 and and $n_D=1$ on LSUN-Churches.

The 3-Stacked MNIST experiment was conducted with an arbitrary EMA of .999. The sampling step size was broadly swept upon. Following \citep{song2020improved} hyperparameter recommendations, we obtained $\sigma_1=50$, $L=200$.

For the synthetic 2D experiments, we used Langevin sampling with $n_\sigma=10$, $\epsilon = .2$, and $\sigma=.1$.

\section{Supplementary experiments}
\label{app:B}

Due to limited computing resources (4 V100 GPUs), the training of models on FFHQ (70k images) in 256x256 \citep{karras2019style}, with the same setting as previously done by \citet{song2020improved}, was impossible. Using a reduced model yielded very poor results. The adversarial version performed worse than the other; we suspect this was the case due to the mini-batch of size 32, our computational limit, while the BigGAN architecture normally assumes very large batch sizes of 2048 when working with images of that size (256x256 or higher). 

\section{Optimal conditional score function}
\label{app:C}

Recall the loss from Eq. \ref{eqn:score}
\[\mathcal{L}[s] = \frac{1}{2L}\sum_{i=1}^L \mathbb{E}_{\boldsymbol{\tilde{x}} \sim q_{\sigma_i}(\boldsymbol{\tilde{x}}|\boldsymbol{x}), \boldsymbol{x} \sim p(\boldsymbol{x})}\left[ \norm{ s(\boldsymbol{\tilde{x}}, \sigma_i) - \nabla_{\boldsymbol{\tilde{x}}} \log q_{\sigma_i}(\boldsymbol{\tilde{x}}) }_2^2 \right] \]
Then, the minimizer $s^*$ of $\mathcal{L}[s]$ would be such that, $s^{*}(\boldsymbol{\tilde{x}}, \sigma_i) = \nabla_{\boldsymbol{\tilde{x}}} \log q_{\sigma_i}(\boldsymbol{\tilde{x}})$ ~ $\forall i \in \{1, ..., L\}$.

\begin{align*}
     \nabla_{\boldsymbol{\tilde{x}}} \log q_\sigma(\boldsymbol{\tilde{x}}) &= \frac{1}{q_\sigma(\boldsymbol{\tilde{x}})} \nabla_{\boldsymbol{\tilde{x}}} \int p(\boldsymbol{x}) q_\sigma(\boldsymbol{\tilde{x}} | \boldsymbol{x}) d\boldsymbol{x} \\
    &= \frac{1}{q_\sigma(\boldsymbol{\tilde{x}})} \int p(\boldsymbol{x}) \nabla_{\boldsymbol{\tilde{x}}} q_\sigma(\boldsymbol{\tilde{x}} | \boldsymbol{x}) d\boldsymbol{x} \\
    &= \frac{1}{q_\sigma(\boldsymbol{\tilde{x}})} \int p(\boldsymbol{x}) q_{\sigma}(\boldsymbol{\tilde{x}}|\boldsymbol{x}) \nabla_{\boldsymbol{\tilde{x}}} \log q_{\sigma}(\boldsymbol{\tilde{x}}|\boldsymbol{x}) d\boldsymbol{x} \\ &= \frac{1}{q_\sigma(\boldsymbol{\tilde{x}})} \int p(\boldsymbol{x}) q_{\sigma}(\boldsymbol{\tilde{x}}|\boldsymbol{x}) \frac{\boldsymbol{x}-\boldsymbol{\tilde{x}}}{\sigma^2} d\boldsymbol{x} \\
    &= \frac{1}{q_\sigma(\boldsymbol{\tilde{x}})} \frac{1}{\sigma^2} \left(\int q_\sigma(\boldsymbol{x} | \boldsymbol{\tilde{x}}) q_\sigma(\boldsymbol{\tilde{x}}) \boldsymbol{x} d\boldsymbol{x} - \boldsymbol{\tilde{x}} q_\sigma(\boldsymbol{\tilde{x}})\right) \\
    &= \frac{\mathbb{E}_{\boldsymbol{x} \sim q_\sigma(\boldsymbol{x} | \boldsymbol{\tilde{x}})}[\boldsymbol{x}] - \boldsymbol{\tilde{x}}}{\sigma^2}
\end{align*}

making use of the fact that  $q_\sigma(\boldsymbol{\boldsymbol{\tilde{x}}}| \boldsymbol{x}) = \mathcal{N}(\boldsymbol{\boldsymbol{\tilde{x}}} | \boldsymbol{x}, \sigma^2 \boldsymbol{I})$, $q_\sigma(\boldsymbol{\boldsymbol{\tilde{x}}}) \triangleq \int p(\boldsymbol{x})q_\sigma(\boldsymbol{\boldsymbol{\tilde{x}}}|\boldsymbol{x})d\boldsymbol{x}$ and that 
$q_\sigma(\boldsymbol{x} | \boldsymbol{\boldsymbol{\tilde{x}}}) = \frac{p(\boldsymbol{x})q_\sigma(\boldsymbol{\boldsymbol{\tilde{x}}}|\boldsymbol{x})}{q_\sigma(\boldsymbol{\boldsymbol{\tilde{x}}})}$

\section{Optimal unconditional score function}
\label{app:D}

As the explicit optimal score function is obtained in Appendix \ref{app:C} for the conditional case, a similar result can be obtained for the unconditional case. Recall the loss from Eq. \ref{eqn:score}

\begin{align}
\mathcal{L}[s] &= \frac{1}{2L} \sum_{i=1}^L  \mathbb{E}_{\boldsymbol{\tilde{x}} \sim q_{\sigma_i}(\boldsymbol{\tilde{x}}|\boldsymbol{x}), \boldsymbol{x} \sim p(\boldsymbol{x})}\left[ \norm{s(\boldsymbol{\tilde{x}}) + \frac{(\boldsymbol{\tilde{x}} - \boldsymbol{x})}{\sigma_i} }_2^2 \right] \\
&= \frac{1}{2} \mathbb{E}_{\boldsymbol{\tilde{x}} \sim q_{\sigma}(\boldsymbol{\tilde{x}}|\boldsymbol{x}), \boldsymbol{x} \sim p(\boldsymbol{x}), \sigma \sim p(\sigma)}\left[ \norm{s(\boldsymbol{\tilde{x}}) + \frac{(\boldsymbol{\tilde{x}} - \boldsymbol{x})}{\sigma} }_2^2 \right],
\end{align}
where $p(\sigma)$ is chosen to be a discrete uniform distribution over a specific set of values. We use this expectation formulation over $\sigma$ to obtain a more general result; the choice of $p(\sigma)$ is not important for this derivation.

Solving with calculus of variations, we get:
\begin{align*} \frac{\partial \mathcal{L}}{\partial s} = &\int \int q_{\sigma}(\boldsymbol{\tilde{x}},\boldsymbol{x},\sigma)\left(s(\boldsymbol{\tilde{x}}) + \frac{\boldsymbol{\tilde{x}} - \boldsymbol{x}}{\sigma}\right)d\boldsymbol{x} d\sigma= 0 \\
\iff &s(\boldsymbol{\tilde{x}})q(\boldsymbol{\tilde{x}}) = \int \int q_{\sigma}(\boldsymbol{\tilde{x}},\boldsymbol{x})p(\sigma) \left(\frac{\boldsymbol{x} - \boldsymbol{\tilde{x}}}{\sigma}\right) d\boldsymbol{x} d\sigma \\
\iff &s(\boldsymbol{\tilde{x}})q(\boldsymbol{\tilde{x}}) = \mathbb{E}_{\sigma \sim p(\sigma)} \left[ \int q_{\sigma}(\boldsymbol{x}|\boldsymbol{\tilde{x}})q(\boldsymbol{\tilde{x}}) \left(\frac{\boldsymbol{x} - \boldsymbol{\tilde{x}}}{\sigma}\right) d\boldsymbol{x} \right] \\
\iff &s(\boldsymbol{\tilde{x}}) = \mathbb{E}_{\sigma \sim p(\sigma)} \left[ \int q_{\sigma}(\boldsymbol{x}|\boldsymbol{\tilde{x}}) \left(\frac{\boldsymbol{x} - \boldsymbol{\tilde{x}}}{\sigma}\right) d\boldsymbol{x} \right] \\
\iff &s(\boldsymbol{\tilde{x}}) = \mathbb{E}_{\sigma \sim p(\sigma)} \left[  \frac{\mathbb{E}_{\boldsymbol{x} \sim q_{\sigma}(\boldsymbol{x} | \boldsymbol{\tilde{x}})}[\boldsymbol{x}] - \boldsymbol{\tilde{x}}}{\sigma} \right] \\
\iff &s(\boldsymbol{\tilde{x}}) = \mathbb{E}_{\sigma \sim p(\sigma)} \left[ s^{*}(\boldsymbol{\tilde{x}}, \sigma) \right]
\end{align*}

Specifically, for our discrete choice of noise levels, we have that the critical point is achieved when $$ s(\boldsymbol{\tilde{x}}) = \frac{1}{L} \sum_{i=1} ^L \left(  \frac{\mathbb{E}_{\boldsymbol{x} \sim q_{\sigma_i}(\boldsymbol{x} | \boldsymbol{\tilde{x}})}[\boldsymbol{x}] - \boldsymbol{\tilde{x}}}{\sigma_i} \right). $$

Let us now consider the scenario where the data distribution is simply a Dirac in $\boldsymbol{x_0}$: $p(\boldsymbol{x}) = \delta_{\boldsymbol{x_0}}$. In that case, the EDS is trivial: $H^*(\boldsymbol{\tilde{x}},\sigma) \triangleq \mathbb{E}_{\boldsymbol{x} \sim q_\sigma(\boldsymbol{x} | \boldsymbol{\tilde{x}})}[\boldsymbol{x}] = \boldsymbol{x_0}$. This gives an interesting insight into the unconditional score network. Letting $\frac{1}{\sigma} = \frac{1}{L} \sum_{i=1}^L \frac{1}{\sigma_i}$, we get:
\begin{align*}
    s(\boldsymbol{\tilde{x}}) &= \frac{\boldsymbol{x}_0 - \boldsymbol{\tilde{x}}}{\sigma} \implies s(\boldsymbol{\tilde{x}}, \sigma_i) = \frac{s(\boldsymbol{\tilde{x}})}{\sigma_i} = \frac{\boldsymbol{x}_0 - \boldsymbol{\tilde{x}}}{\sigma \sigma_i}.
\end{align*}
This should be compared to the true value of the conditional network $s(\boldsymbol{\tilde{x}}, \sigma_i) = \frac{\boldsymbol{x}_0 - \boldsymbol{\tilde{x}}}{\sigma_i^2}$. We see that in the denominator of the unconditional score network, a noise level is replaced by the harmonic mean of noise levels. In particular, for $\sigma_i$ far from its mean, the approximation will get inaccurate:
\begin{itemize}
    \item For large noise values, the unconditional score network will overestimate the true score function, leading to a larger effective step size during sampling.
    \item For small noise values, the unconditional score network will underestimate the true score function, leading to samples not diffusing as much as they should.
\end{itemize}

\section{ALS Non-geometric proof}\label{app:als_bad_proof}
\input{sections/prop1_proof_final.tex}

\section{CAS Geometric proof}\label{app:cas_good_proof}
\input{sections/prop2_proof_alt.tex}

\section{Update rule}
\label{app:update_rule_proof}

%\addtocounter{proposition}{1}
\begin{proposition}\label{prop:ff}
\em Given a noise-conditional score function, the update rules from Algorithm \ref{alg:anneal1} and Algorithm \ref{alg:anneal2} are equivalent to the respective following update rules: \em 
\begin{align*} &\boldsymbol{x} \gets (1-\eta)\boldsymbol{x} + \eta H(\boldsymbol{x},\sigma_i) + \sqrt{2\eta}\sigma_i \boldsymbol{z}  && \text{for } \boldsymbol{z} \sim \mathcal{N}(0, \boldsymbol{I}) ~~ \text{and} ~~ \eta = \frac{\epsilon}{\sigma_L^2} \\
&\boldsymbol{x} \gets (1-\eta)\boldsymbol{x} + \eta H(\boldsymbol{x},\sigma_i) + \beta \sigma_{i+1} \boldsymbol{z}
\end{align*}
\end{proposition}

\begin{proof}
Recall from Algorithm \ref{alg:anneal1} that $\alpha_i = \epsilon \frac{\sigma_i^2}{\sigma_L^2} = \eta \sigma_i^2$. Then, the update rule is as follows:

\begin{align*}
    \boldsymbol{x} &\gets \boldsymbol{{x}} + \alpha_i s(\boldsymbol{{x}}, \sigma_i) + \sqrt{2 \alpha_i}~ \boldsymbol{z}\\ &= \boldsymbol{{x}} + \eta \sigma_i^2 \left(\frac{H(\boldsymbol{x}) - \boldsymbol{{x}}}{\sigma_i^2}\right)+  \sqrt{2 \alpha_i}~ \boldsymbol{z} \\
    &= (1-\eta)\boldsymbol{{x}} + \eta H(\boldsymbol{x}) + \sqrt{2\eta}\sigma_i \boldsymbol{z}
\end{align*}
The same thing can be proven for Algorithm \ref{alg:anneal2} in the very same way.
\end{proof}

\section{Inception Score (IS)}\label{sec:is}

While the FID is improved by applying the EDS to image samples, the Inception Score is not.
Convolutional neural networks suffer from texture bias \citep{DBLP:journals/corr/abs-1811-12231}. Since the IS is built upon convolution layers, this flaw is also strongly present in the metric. Designed to answer the question of how easy it is to recover the class of an image, it tends to bias towards within-class texture similarity \citep{ISnote}.

Since we denoise the final image, we are evaluating the expected lower level of details across all classes. Therefore, the denoiser will confound the textures used by the IS to distinguish between classes, invariably worsening the score. Since the FID has already been shown to be more consistent with the level of noise than the IS \citep{heusel2017gans}, and since ALS methods are particularly prone to inject class-specific imperceptible noise, we would recommend against its use to compare within and between score matching models.

\newpage
\section{Uncurated samples}

\begin{figure}[ht] 
    \centering
    \includegraphics[width=0.6\linewidth]{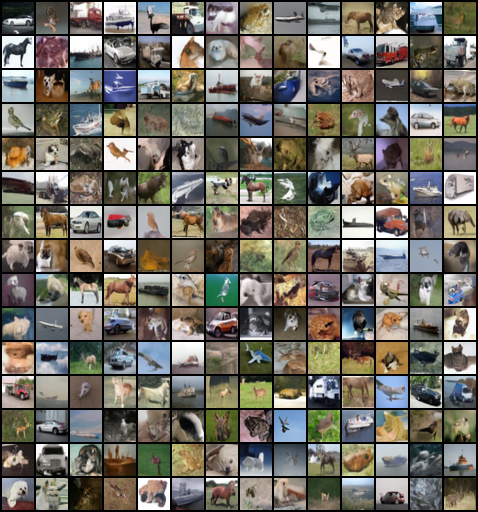}
    \caption{CIFAR-10 Non-adversarial non-consistent $n_\sigma=1$} 
\end{figure}
\begin{figure}[ht] 
    \centering
    \includegraphics[width=0.6\linewidth]{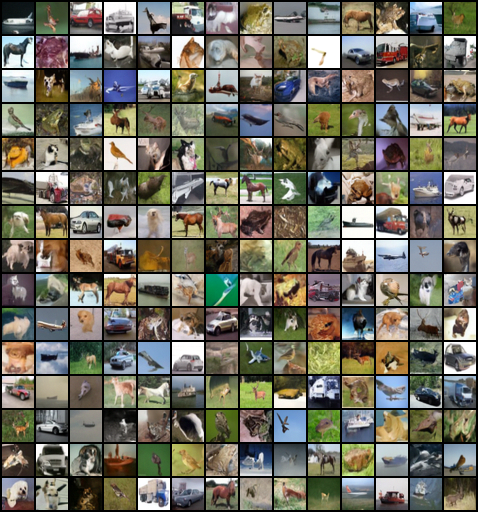}
    \caption{CIFAR-10 Adversarial non-consistent $n_\sigma=1$} 
\end{figure}

\begin{figure}[ht] 
    \centering
    \includegraphics[width=0.65\linewidth]{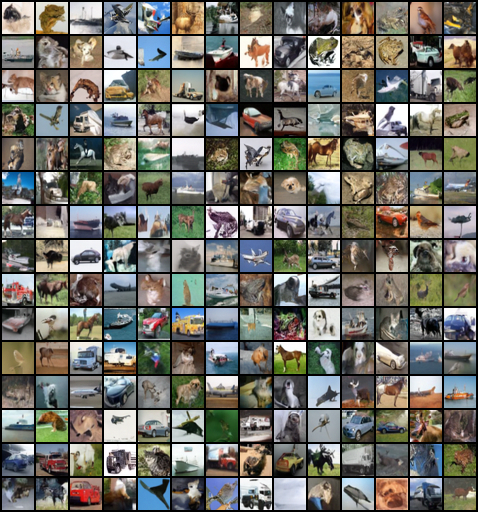}
    \caption{CIFAR-10 Non-adversarial non-consistent $n_\sigma=5$} 
\end{figure}
\begin{figure}[ht] 
    \centering
    \includegraphics[width=0.65\linewidth]{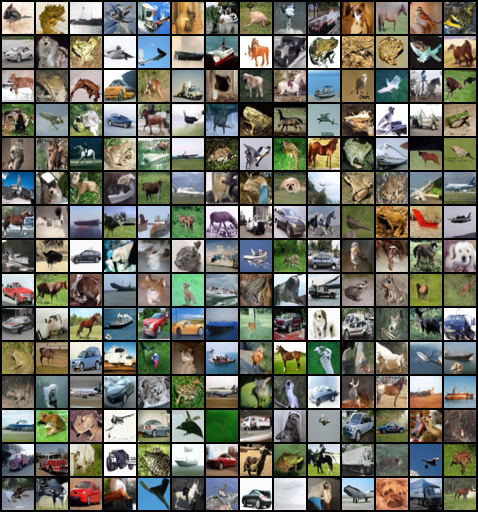}
    \caption{CIFAR-10 Adversarial non-consistent $n_\sigma=5$} 
\end{figure}

\begin{figure}[ht] 
    \centering
    \includegraphics[width=0.65\linewidth]{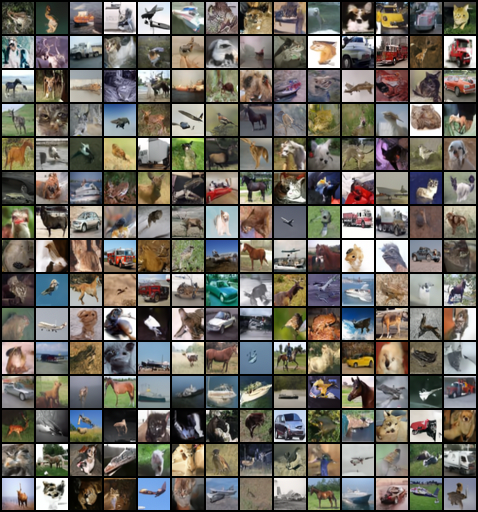}
    \caption{CIFAR-10 Non-adversarial consistent $n_\sigma=1$} 
\end{figure}
\begin{figure}[ht] 
    \centering
    \includegraphics[width=0.65\linewidth]{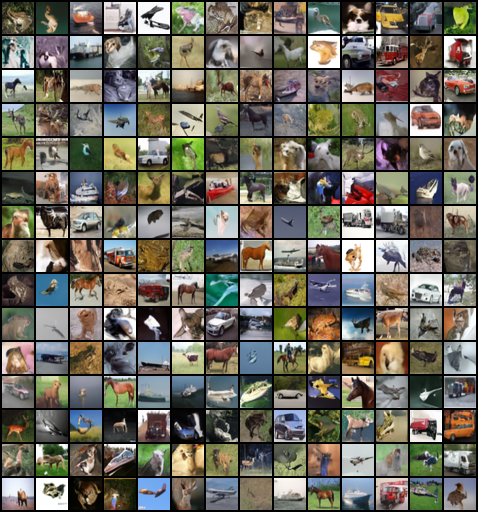}
    \caption{CIFAR-10 Adversarial consistent $n_\sigma=1$} 
\end{figure}

\begin{figure}[ht] 
    \centering
    \includegraphics[width=0.65\linewidth]{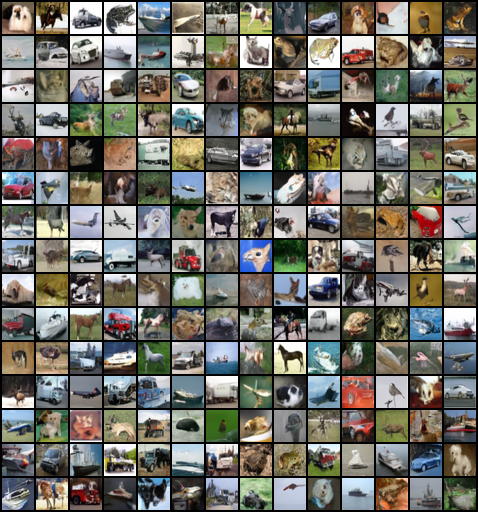}
    \caption{CIFAR-10 Non-adversarial consistent $n_\sigma=5$} 
\end{figure}
\begin{figure}[ht] 
    \centering
    \includegraphics[width=0.65\linewidth]{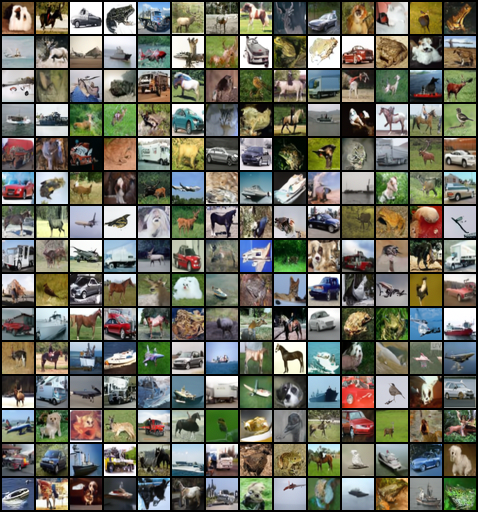}
    \caption{CIFAR-10 Adversarial consistent $n_\sigma=5$} 
\end{figure}

%%%%%%%%%%%%%%%%%%%%%%%%%%%%%%%%%%
%%%%%%%%%%%%%%%%%%%%%%%%%%%%%%%%%%

\begin{figure}[ht] 
    \centering
    \includegraphics[width=0.75\linewidth]{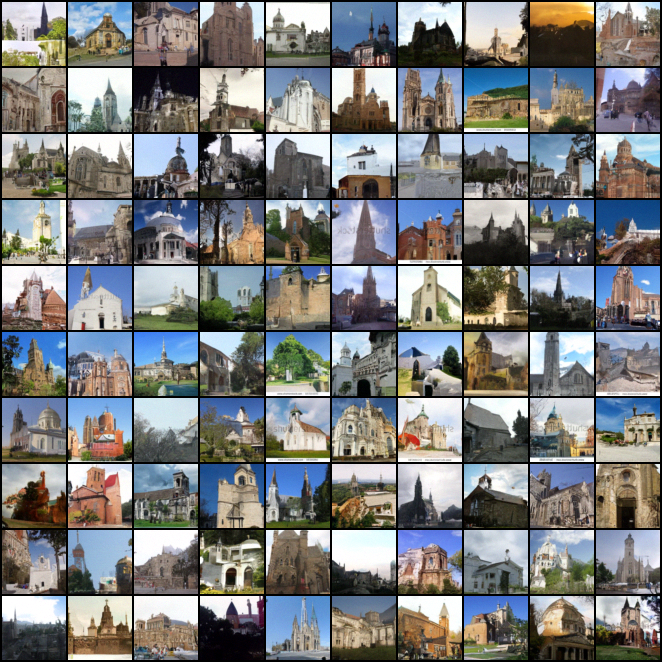}
    \caption{LSUN-Churches Non-adversarial non-consistent $n_\sigma=1$} 
\end{figure}
\begin{figure}[ht] 
    \centering
    \includegraphics[width=0.75\linewidth]{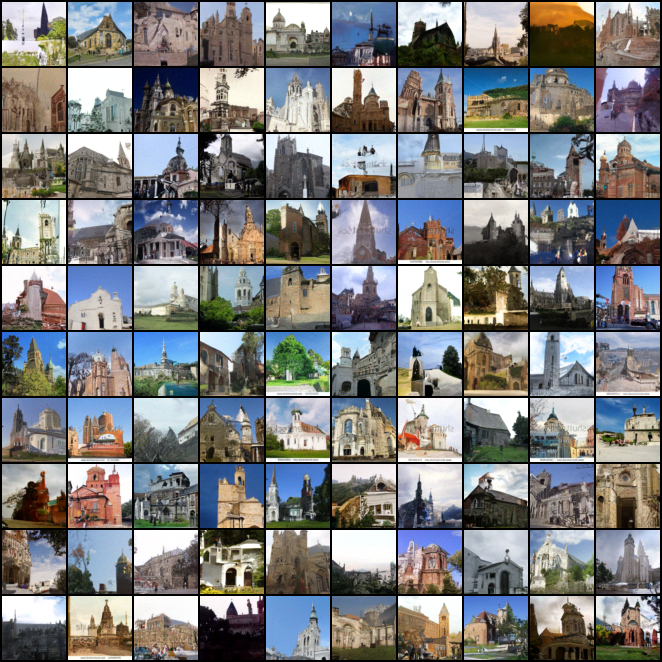}
    \caption{LSUN-Churches Adversarial non-consistent $n_\sigma=1$} 
\end{figure}

\begin{figure}[ht] 
    \centering
    \includegraphics[width=0.75\linewidth]{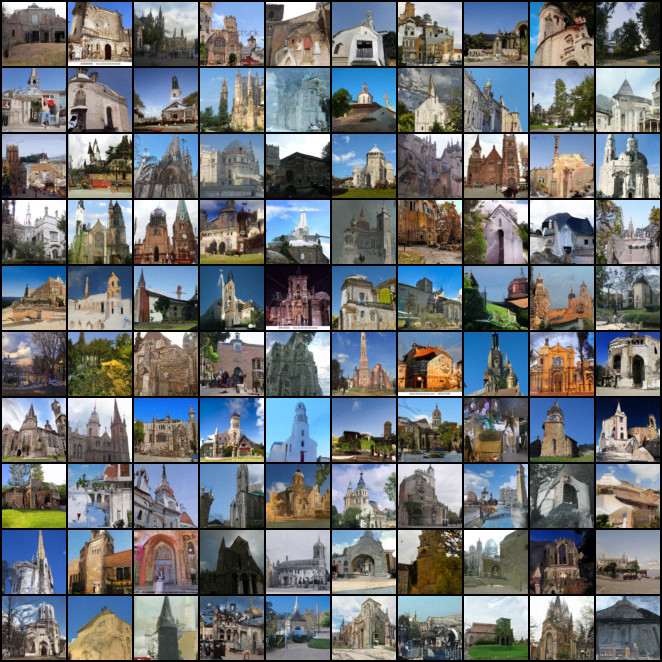}
    \caption{LSUN-Churches Non-adversarial non-consistent $n_\sigma=5$} 
\end{figure}
\begin{figure}[ht] 
    \centering
    \includegraphics[width=0.75\linewidth]{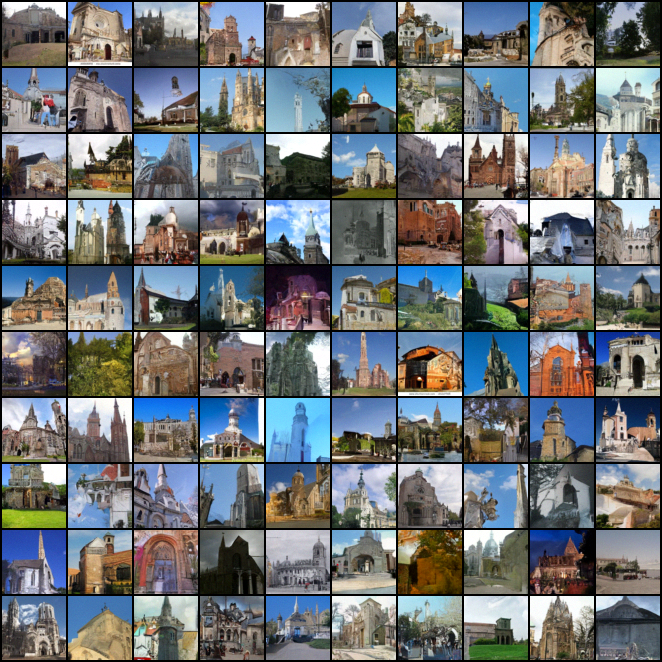}
    \caption{LSUN-Churches Adversarial non-consistent $n_\sigma=5$} 
\end{figure}

\begin{figure}[ht] 
    \centering
    \includegraphics[width=0.75\linewidth]{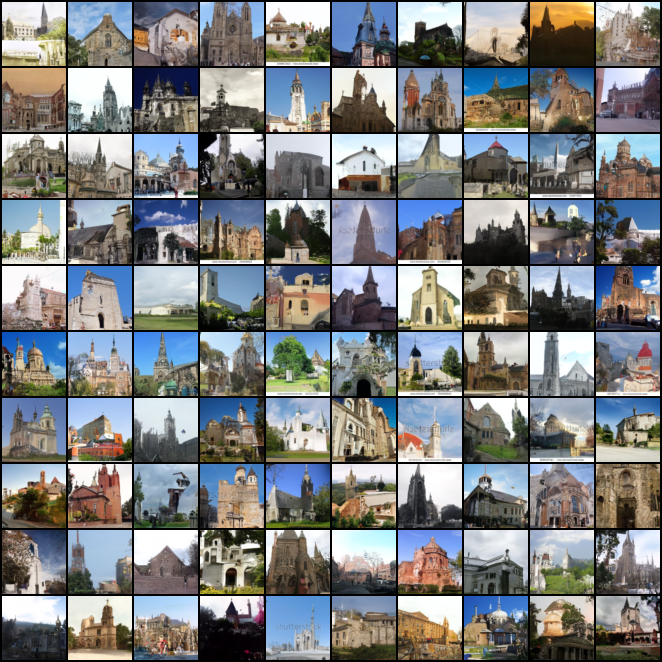}
    \caption{LSUN-Churches Non-adversarial consistent $n_\sigma=1$} 
\end{figure}
\begin{figure}[ht] 
    \centering
    \includegraphics[width=0.75\linewidth]{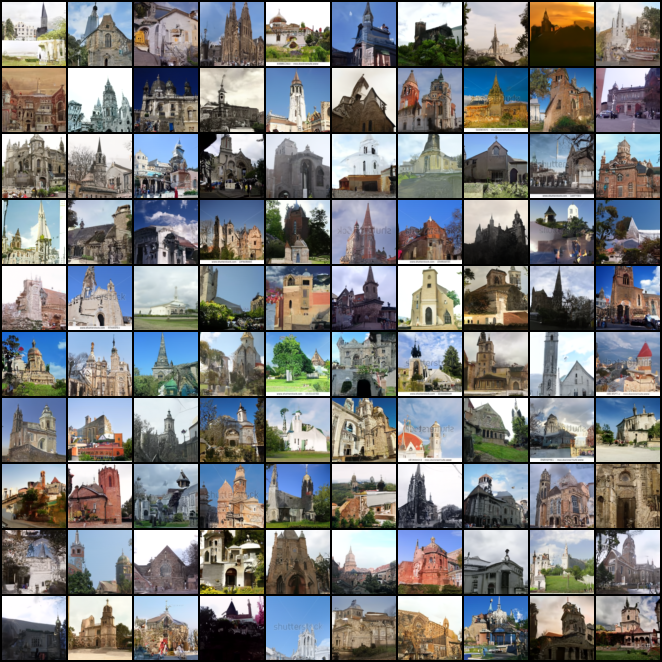}
    \caption{LSUN-Churches Adversarial consistent $n_\sigma=1$} 
\end{figure}

\begin{figure}[ht] 
    \centering
    \includegraphics[width=0.75\linewidth]{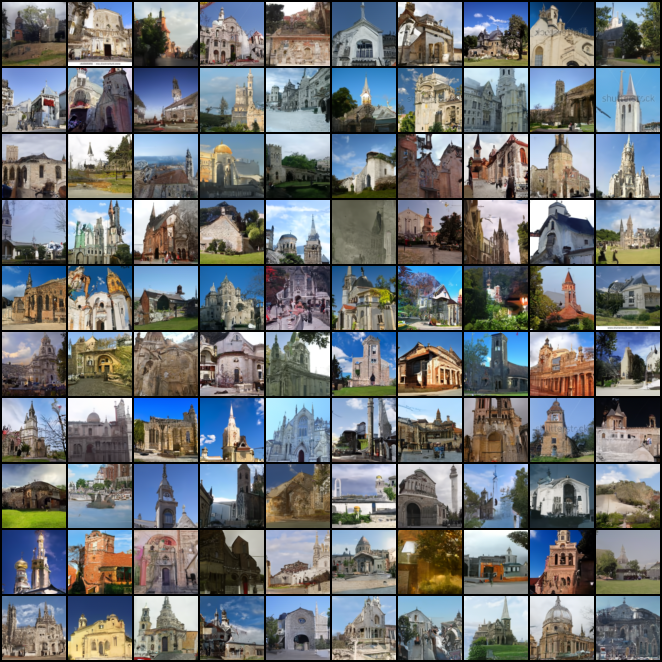}
    \caption{LSUN-Churches Non-adversarial consistent $n_\sigma=5$} 
\end{figure}
\begin{figure}[ht] 
    \centering
    \includegraphics[width=0.75\linewidth]{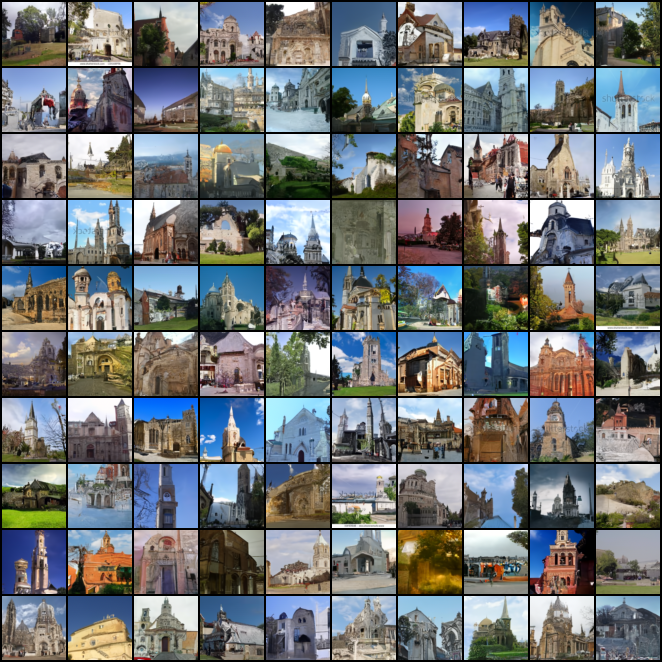}
    \caption{LSUN-Churches Adversarial consistent $n_\sigma=5$} 
\end{figure}

%%%%%%%%%%%%%%%%%%%%%%%%%%%%%%%%%%

\begin{figure}[ht] 
    \centering
    \includegraphics[width=0.65\linewidth]{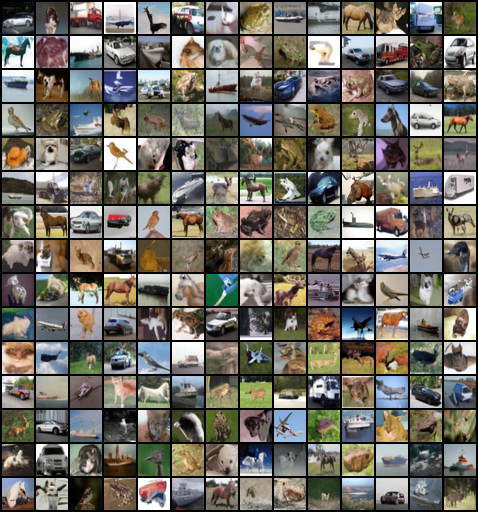}
    \caption{\citet{ho2020denoising} network architecture with CIFAR-10 Non-adversarial non-consistent $n_\sigma=1$} 
\end{figure}
\begin{figure}[ht] 
    \centering
    \includegraphics[width=0.65\linewidth]{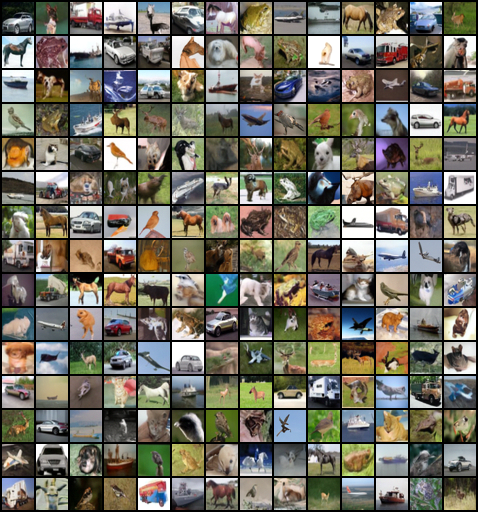}
    \caption{\citet{ho2020denoising} network architecture with CIFAR-10 Adversarial non-consistent $n_\sigma=1$} 
\end{figure}

\begin{figure}[ht] 
    \centering
    \includegraphics[width=0.65\linewidth]{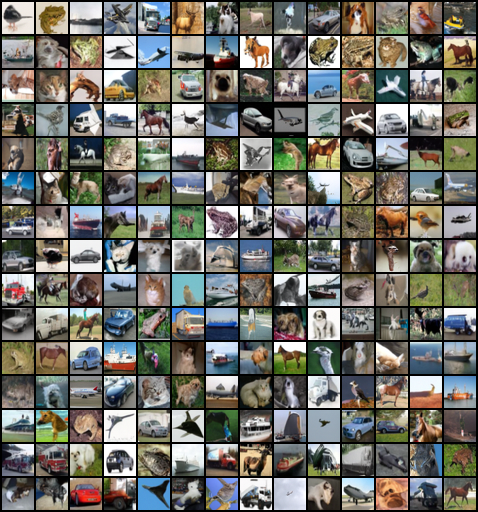}
    \caption{\citet{ho2020denoising} network architecture with CIFAR-10 Non-adversarial non-consistent $n_\sigma=5$} 
\end{figure}
\begin{figure}[ht] 
    \centering
    \includegraphics[width=0.65\linewidth]{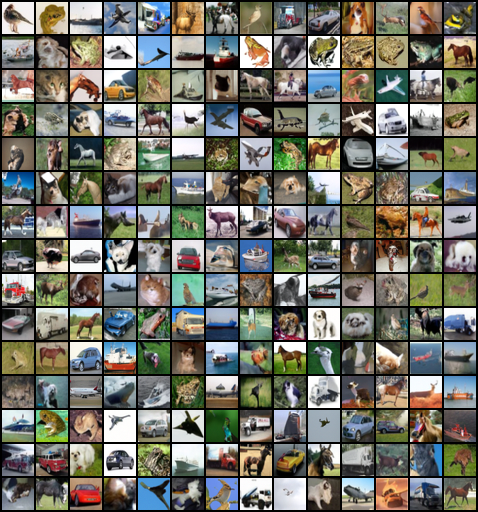}
    \caption{\citet{ho2020denoising} network architecture with CIFAR-10 Adversarial non-consistent $n_\sigma=5$} 
\end{figure}

\begin{figure}[ht] 
    \centering
    \includegraphics[width=0.65\linewidth]{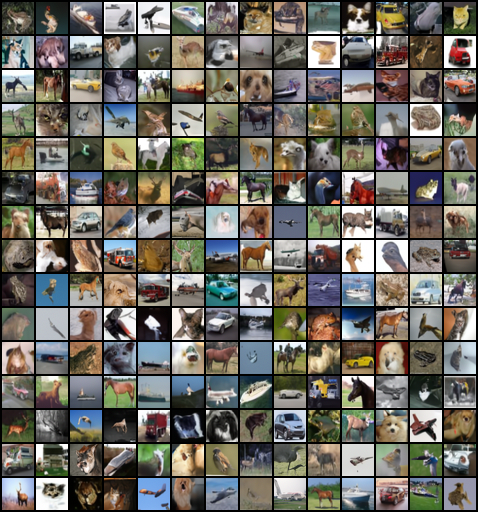}
    \caption{\citet{ho2020denoising} network architecture with CIFAR-10 Non-adversarial consistent $n_\sigma=1$} 
\end{figure}
\begin{figure}[ht] 
    \centering
    \includegraphics[width=0.65\linewidth]{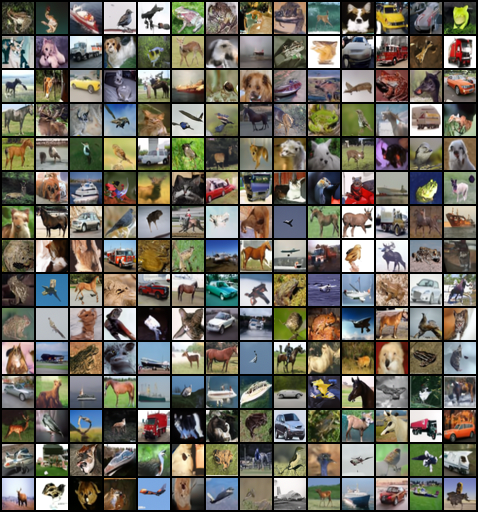}
    \caption{\citet{ho2020denoising} network architecture with CIFAR-10 Adversarial consistent $n_\sigma=1$} 
\end{figure}

\begin{figure}[ht] 
    \centering
    \includegraphics[width=0.65\linewidth]{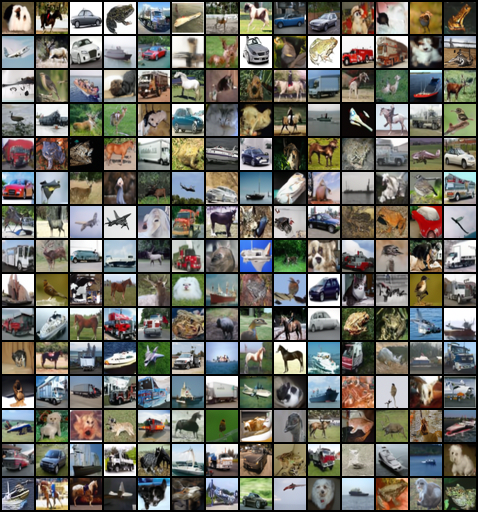}
    \caption{\citet{ho2020denoising} network architecture with CIFAR-10 Non-adversarial consistent $n_\sigma=5$} 
\end{figure}
\begin{figure}[ht] 
    \centering
    \includegraphics[width=0.65\linewidth]{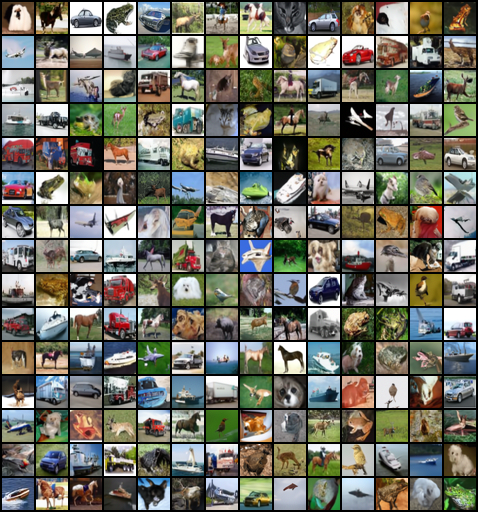}
    \caption{\citet{ho2020denoising} network architecture with CIFAR-10 Adversarial consistent $n_\sigma=5$} 
\end{figure}

\end{document}

%% file: sections/authors_hardcoded.tex
\renewcommand{\thefootnote}{\fnsymbol{footnote}}

\vspace{-1cm}
\hspace{-1cm}
\setlength{\tabcolsep}{30pt}
\begin{tabular}{ll}
\textbf{Alexia Jolicoeur-Martineau}\textsuperscript{$\dagger$} & \textbf{Rémi Piché-Taillefer}\textsuperscript{$\dagger$} \\
    \normalfont Department of Computer Science & \normalfont Department of Computer Science \\
    \normalfont University of Montreal & \normalfont University of Montreal \\  \vspace{5pt} \\
    \textbf{Ioannis Mitliagkas} & \textbf{Rémi Tachet des Combes} \\ \normalfont Department of Computer Science & \normalfont Microsoft Research Lab \\
    \normalfont University of Montreal & \normalfont Microsoft Research Montreal
\end{tabular}
\vspace{1cm}
\setlength{\tabcolsep}{6pt}
\footnote[0]{$\dagger$ First two authors have equal contributions}

%% file: sections/prop1_proof_final.tex
\addtocounter{proposition}{-3}
\begin{proposition}
\em 
Let $s^*$ be the optimal score function from Eq. \ref{eqn:eds}. Following the sampling described in Algorithm \ref{alg:anneal1}, the variance of the noise component in the sample $\boldsymbol{x}$ will remain greater than $\sigma_t^2$ at every step $t$.\em 
\end{proposition}

\begin{proof}
Assume at the start of an iteration of Langevin Sampling that the point $\boldsymbol{x}$ is comprised of an \em image component \em and of a \em noise component \em denoted $v_0 \boldsymbol{z}_0$ for $\boldsymbol{z}_0 \sim \mathcal{N}(0, \boldsymbol{I})$. We assume from the proposition statement that the Langevin Sampling is performed at the level of variance $\sigma_t^2$, meaning the update rule is as follows:
\vspace{-0.1cm}
\begin{align*}
    \boldsymbol{x} \gets & \boldsymbol{x} + \eta \sigma_t^2 s^*(\boldsymbol{x}, \sigma_t) + \sigma_t\sqrt{2\eta}\boldsymbol{z} \\
\intertext{ for $\boldsymbol{z} \sim \mathcal{N}(0, \boldsymbol{I}) \text{ and } 0 < \eta < 1$. From Eq. \ref{eqn:eds}, we get:}
\boldsymbol{x} \gets & \boldsymbol{x} + \eta (\mathbb{E}_{\boldsymbol{x'}\sim q_{v_0}(\boldsymbol{x'}|\boldsymbol{x})}[\boldsymbol{x'}] - \boldsymbol{x}) + \sigma_t\sqrt{2\eta}\boldsymbol{z} \\
= & (1-\eta)\boldsymbol{x} + \eta \mathbb{E}_{\boldsymbol{x'}\sim q_{v_0}(\boldsymbol{x'}|\boldsymbol{x})}[\boldsymbol{x'}] + \sigma_t\sqrt{2\eta}\boldsymbol{z}.
\intertext{The noise component of $(1-\eta)\boldsymbol{x}$ and $\sigma_t\sqrt{2\eta}\boldsymbol{z}$ can then be summed as:}
& (1 - \eta)v_0\boldsymbol{z}_0 + \sigma_t \sqrt{2\eta} \boldsymbol{z}.
\intertext{
Making use of the fact that both sources of noise follow independent normal distributions, the sum will be normally distributed, centered and of variance $v_1^2$, with}
    v_1^2 = & v_0^2(1-\eta)^2 + 2\eta \sigma_t^2.
\intertext{
Applying the same steps allows us to examine the variance after multiple Langevin Sampling iterations}
v_2^2 = & v_1^2 (1-\eta)^2 + 2\eta \sigma_t^2 \\
    = & v_0^2(1-\eta)^4 + 2\eta \sigma_t^2 + 2\eta \sigma_t^2 (1-\eta)^2 \\
    \vdots  \\
    v_n^2 = & v_0^2(1-\eta)^{2n}+ 2\eta \sigma_t^2\sum_{i=0}^{n-1}(1 - \eta)^{2i} \\ 
    = & v_0^2(1-\eta)^{2n}+\frac{2\sigma_t^2}{2-\eta}(1 - (1 -\eta)^{2n}).
\intertext{
From there, the two following statements can be obtained:}
    (1) ~~& \lim_{n \to \infty} v_n = \sigma_t\sqrt{\frac{2}{2-\eta}} > \sigma_t \\
    (2) ~~& \frac{dv_n}{dn} < 0 \iff \eta < 2 - \frac{2\sigma_t^2}{v_0^2}.
\end{align*}
From these observations, we understand that $v_n^2$ is monotonically decreasing (under conditions generally respected in practice) but converges to a point superior to $\sigma_t^2$ after an infinite number of Langevin Sampling steps. 
We then conclude that for all $n$, the variance of the noise component in the sample will always exceed $\sigma_t^2$. We also note that this will be true across the full sampling, at every step $t$.

In the particular case where $\sigma_t$ corresponds to a geometrically decreasing series, it means that even given an optimal score function, the standard deviation of the noise component cannot follow its prescribed schedule.

\end{proof}

%% file: sections/prop2_proof_alt.tex
\begin{proposition}
\em Let $s^*$ be the optimal score function from Eq. \ref{eqn:eds}. Following the sampling described in Algorithm \ref{alg:anneal2}, the variance of the noise component in the sample $\boldsymbol{x}$ will consistently be equal to $\sigma_t^2$ at every step $t$. \em 
\end{proposition}

\begin{proof}
Let us first define $\beta$ to be equal to $\sqrt{1 - (1-\eta)^2/\gamma^2}$ with $\eta = \epsilon / \sigma_L^2$ and $\gamma$ defined as Eq. \ref{sigma_set}.
At the start of Algorithm \ref{alg:anneal2}, assume that $\boldsymbol{x}$ is comprised of an \em image component \em and a \em noise component \em denoted $\sigma_0 \boldsymbol{z}_0$, where $\boldsymbol{z}_0 \sim \mathcal{N}(0, \boldsymbol{I})$ and $\sigma_0 = \sigma_1 / \gamma$. We will proceed by induction to show that the noise component at step $t$ will be a Gaussian of variance $\sigma_t^2$ for every $t$.
\vspace{-0.1cm}
\begin{align*}
\intertext{The first induction step is trivial. Assume the noise component of $\boldsymbol{x}_t$ to be $\sigma_t \boldsymbol{z}_t$, where $\boldsymbol{z}_t \sim \mathcal{N}(0, \boldsymbol{I})$. Following Algorithm \ref{alg:anneal2}, the update step will be:}
    & \boldsymbol{x}_{t+1} \gets \boldsymbol{x}_t + \eta \sigma_{t}^2 s^*(\boldsymbol{x}_t, \sigma_t) + \sigma_{t+1}\beta\boldsymbol{z},
\intertext{with $\boldsymbol{z} \sim \mathcal{N}(0, \boldsymbol{I}) \text{ and } 0 < \eta < 1$. From Eq. \ref{eqn:eds}, we get}
\boldsymbol{x}_{t+1} \gets & \boldsymbol{x}_t + \eta( \mathbb{E}_{\boldsymbol{x}\sim q_{\sigma_t}(\boldsymbol{x}|\boldsymbol{x}_t)}[\boldsymbol{x}] - \boldsymbol{x}_t) + \sigma_{t+1}\beta\boldsymbol{z} \\
= & (1-\eta)\boldsymbol{x}_t + \eta \mathbb{E}_{\boldsymbol{x}\sim q_{\sigma_t}(\boldsymbol{x}|\boldsymbol{x}_t)}[\boldsymbol{x}] + \sigma_{t+1}\beta\boldsymbol{z}.
\intertext{The noise component from $(1-\eta)\boldsymbol{x}_t$ and $\sigma_{t+1}\beta\boldsymbol{z}$ can then be summed as:}
& \sigma_t(1 - \eta)\boldsymbol{z}_t + \sigma_{t+1}\beta\boldsymbol{z}.
\intertext{Making use of the fact that both sources of noise follow independent normal distributions, the sum will be normally distributed, centered and of variance:}
& \sigma_t^2(1-\eta)^2 + \sigma_{t+1}^2\beta^2
= \sigma_{t+1}^2\left[ \left(\frac{1-\eta}{\gamma}\right)^2 + \beta^2 \right]
= \sigma_{t+1}^2.
\end{align*}
By induction, the noise component of the sample $\boldsymbol{x}$ will follow a Gaussian distribution of variance $\sigma_i^2 ~~ \forall i \in \{0, ..., L\}$. In the particular case where $\sigma_i$ corresponds to a geometrically decreasing series, it means that given an optimal score function, the standard deviation of the noise component will follow its prescribed schedule.
\end{proof}

%% file: main.bbl
\begin{thebibliography}{57}
\providecommand{\natexlab}[1]{#1}
\providecommand{\url}[1]{\texttt{#1}}
\expandafter\ifx\csname urlstyle\endcsname\relax
  \providecommand{\doi}[1]{doi: #1}\else
  \providecommand{\doi}{doi: \begingroup \urlstyle{rm}\Url}\fi

\bibitem[Song and Ermon(2019)]{song2019generative}
Yang Song and Stefano Ermon.
\newblock Generative modeling by estimating gradients of the data distribution.
\newblock In \emph{Advances in Neural Information Processing Systems}, pages
  11918--11930, 2019.

\bibitem[Hyv{\"a}rinen(2005)]{hyvarinen2005estimation}
Aapo Hyv{\"a}rinen.
\newblock Estimation of non-normalized statistical models by score matching.
\newblock \emph{Journal of Machine Learning Research}, 6\penalty0
  (Apr):\penalty0 695--709, 2005.

\bibitem[Vincent(2011)]{vincent2011connection}
Pascal Vincent.
\newblock A connection between score matching and denoising autoencoders.
\newblock \emph{Neural computation}, 23\penalty0 (7):\penalty0 1661--1674,
  2011.

\bibitem[Raphan and Simoncelli(2011)]{raphan2011least}
Martin Raphan and Eero~P Simoncelli.
\newblock Least squares estimation without priors or supervision.
\newblock \emph{Neural computation}, 23\penalty0 (2):\penalty0 374--420, 2011.

\bibitem[Welling and Teh(2011)]{welling2011bayesian}
Max Welling and Yee~W Teh.
\newblock Bayesian learning via stochastic gradient langevin dynamics.
\newblock In \emph{Proceedings of the 28th international conference on machine
  learning (ICML-11)}, pages 681--688, 2011.

\bibitem[Roberts et~al.(1996)Roberts, Tweedie, et~al.]{roberts1996exponential}
Gareth~O Roberts, Richard~L Tweedie, et~al.
\newblock Exponential convergence of langevin distributions and their discrete
  approximations.
\newblock \emph{Bernoulli}, 2\penalty0 (4):\penalty0 341--363, 1996.

\bibitem[Song and Ermon(2020)]{song2020improved}
Yang Song and Stefano Ermon.
\newblock Improved techniques for training score-based generative models.
\newblock \emph{arXiv preprint arXiv:2006.09011}, 2020.

\bibitem[Heusel et~al.(2017)Heusel, Ramsauer, Unterthiner, Nessler, and
  Hochreiter]{heusel2017gans}
Martin Heusel, Hubert Ramsauer, Thomas Unterthiner, Bernhard Nessler, and Sepp
  Hochreiter.
\newblock Gans trained by a two time-scale update rule converge to a local nash
  equilibrium.
\newblock In \emph{Advances in Neural Information Processing Systems}, pages
  6626--6637, 2017.

\bibitem[Goodfellow et~al.(2014)Goodfellow, Pouget-Abadie, Mirza, Xu,
  Warde-Farley, Ozair, Courville, and Bengio]{GAN}
Ian Goodfellow, Jean Pouget-Abadie, Mehdi Mirza, Bing Xu, David Warde-Farley,
  Sherjil Ozair, Aaron Courville, and Yoshua Bengio.
\newblock Generative adversarial nets.
\newblock In Z.~Ghahramani, M.~Welling, C.~Cortes, N.~D. Lawrence, and K.~Q.
  Weinberger, editors, \emph{Advances in Neural Information Processing Systems
  27}, pages 2672--2680. Curran Associates, Inc., 2014.
\newblock URL
  \url{http://papers.nips.cc/paper/5423-generative-adversarial-nets.pdf}.

\bibitem[Brock et~al.(2018)Brock, Donahue, and Simonyan]{brock2018large}
Andrew Brock, Jeff Donahue, and Karen Simonyan.
\newblock Large scale gan training for high fidelity natural image synthesis.
\newblock \emph{arXiv preprint arXiv:1809.11096}, 2018.

\bibitem[Karras et~al.(2017)Karras, Aila, Laine, and
  Lehtinen]{karras2017progressive}
Tero Karras, Timo Aila, Samuli Laine, and Jaakko Lehtinen.
\newblock Progressive growing of gans for improved quality, stability, and
  variation.
\newblock \emph{arXiv preprint arXiv:1710.10196}, 2017.

\bibitem[Karras et~al.(2019)Karras, Laine, and Aila]{karras2019style}
Tero Karras, Samuli Laine, and Timo Aila.
\newblock A style-based generator architecture for generative adversarial
  networks.
\newblock In \emph{Proceedings of the IEEE Conference on Computer Vision and
  Pattern Recognition}, pages 4401--4410, 2019.

\bibitem[Karras et~al.(2020)Karras, Laine, Aittala, Hellsten, Lehtinen, and
  Aila]{karras2020analyzing}
Tero Karras, Samuli Laine, Miika Aittala, Janne Hellsten, Jaakko Lehtinen, and
  Timo Aila.
\newblock Analyzing and improving the image quality of stylegan.
\newblock In \emph{Proceedings of the IEEE/CVF Conference on Computer Vision
  and Pattern Recognition}, pages 8110--8119, 2020.

\bibitem[Radford et~al.(2015{\natexlab{a}})Radford, Metz, and
  Chintala]{radford2015unsupervised}
Alec Radford, Luke Metz, and Soumith Chintala.
\newblock Unsupervised representation learning with deep convolutional
  generative adversarial networks.
\newblock \emph{arXiv preprint arXiv:1511.06434}, 2015{\natexlab{a}}.

\bibitem[Miyato et~al.(2018)Miyato, Kataoka, Koyama, and
  Yoshida]{miyato2018spectral}
Takeru Miyato, Toshiki Kataoka, Masanori Koyama, and Yuichi Yoshida.
\newblock Spectral normalization for generative adversarial networks.
\newblock \emph{arXiv preprint arXiv:1802.05957}, 2018.

\bibitem[Jolicoeur-Martineau(2018)]{jolicoeur2018relativistic}
Alexia Jolicoeur-Martineau.
\newblock The relativistic discriminator: a key element missing from standard
  gan.
\newblock \emph{arXiv preprint arXiv:1807.00734}, 2018.

\bibitem[Zhang et~al.(2019)Zhang, Goodfellow, Metaxas, and
  Odena]{zhang2019self}
Han Zhang, Ian Goodfellow, Dimitris Metaxas, and Augustus Odena.
\newblock Self-attention generative adversarial networks.
\newblock In \emph{International Conference on Machine Learning}, pages
  7354--7363, 2019.

\bibitem[Kodali et~al.(2017)Kodali, Abernethy, Hays, and
  Kira]{kodali2017convergence}
Naveen Kodali, Jacob Abernethy, James Hays, and Zsolt Kira.
\newblock On convergence and stability of gans.
\newblock \emph{arXiv preprint arXiv:1705.07215}, 2017.

\bibitem[Gulrajani et~al.(2017)Gulrajani, Ahmed, Arjovsky, Dumoulin, and
  Courville]{WGAN-GP}
Ishaan Gulrajani, Faruk Ahmed, Martin Arjovsky, Vincent Dumoulin, and Aaron~C
  Courville.
\newblock Improved training of wasserstein gans.
\newblock In I.~Guyon, U.~V. Luxburg, S.~Bengio, H.~Wallach, R.~Fergus,
  S.~Vishwanathan, and R.~Garnett, editors, \emph{Advances in Neural
  Information Processing Systems 30}, pages 5767--5777. Curran Associates,
  Inc., 2017.
\newblock URL
  \url{http://papers.nips.cc/paper/7159-improved-training-of-wasserstein-gans.pdf}.

\bibitem[Arjovsky et~al.(2017)Arjovsky, Chintala, and Bottou]{WGAN}
Martin Arjovsky, Soumith Chintala, and L{\'e}on Bottou.
\newblock Wasserstein generative adversarial networks.
\newblock In \emph{International Conference on Machine Learning}, pages
  214--223, 2017.

\bibitem[Jolicoeur-Martineau and Mitliagkas(2019)]{jolicoeur2019connections}
Alexia Jolicoeur-Martineau and Ioannis Mitliagkas.
\newblock Connections between support vector machines, wasserstein distance and
  gradient-penalty gans.
\newblock \emph{arXiv preprint arXiv:1910.06922}, 2019.

\bibitem[Ho et~al.(2020)Ho, Jain, and Abbeel]{ho2020denoising}
Jonathan Ho, Ajay Jain, and Pieter Abbeel.
\newblock Denoising diffusion probabilistic models.
\newblock \emph{arXiv preprint arxiv:2006.11239}, 2020.

\bibitem[Sohl-Dickstein et~al.(2015)Sohl-Dickstein, Weiss, Maheswaranathan, and
  Ganguli]{sohl2015deep}
Jascha Sohl-Dickstein, Eric~A Weiss, Niru Maheswaranathan, and Surya Ganguli.
\newblock Deep unsupervised learning using nonequilibrium thermodynamics.
\newblock \emph{arXiv preprint arXiv:1503.03585}, 2015.

\bibitem[Goyal et~al.(2017)Goyal, Ke, Ganguli, and
  Bengio]{goyal2017variational}
Anirudh Goyal Alias~Parth Goyal, Nan~Rosemary Ke, Surya Ganguli, and Yoshua
  Bengio.
\newblock Variational walkback: Learning a transition operator as a stochastic
  recurrent net.
\newblock In \emph{Advances in Neural Information Processing Systems}, pages
  4392--4402, 2017.

\bibitem[Salimans et~al.(2017)Salimans, Karpathy, Chen, and
  Kingma]{salimans2017pixelcnn++}
Tim Salimans, Andrej Karpathy, Xi~Chen, and Diederik~P Kingma.
\newblock Pixelcnn++: Improving the pixelcnn with discretized logistic mixture
  likelihood and other modifications.
\newblock \emph{arXiv preprint arXiv:1701.05517}, 2017.

\bibitem[Lin et~al.(2017{\natexlab{a}})Lin, Milan, Shen, and
  Reid]{lin2017refinenet}
Guosheng Lin, Anton Milan, Chunhua Shen, and Ian Reid.
\newblock Refinenet: Multi-path refinement networks for high-resolution
  semantic segmentation.
\newblock In \emph{Proceedings of the IEEE conference on computer vision and
  pattern recognition}, pages 1925--1934, 2017{\natexlab{a}}.

\bibitem[Li et~al.(2019)Li, Chen, and Sommer]{li2019learning}
Zengyi Li, Yubei Chen, and Friedrich~T Sommer.
\newblock Learning energy-based models in high-dimensional spaces with
  multi-scale denoising score matching.
\newblock \emph{arXiv}, pages arXiv--1910, 2019.

\bibitem[Lim et~al.(2020)Lim, Courville, Pal, and Huang]{AR-DAE}
Jae~Hyun Lim, Aaron~C. Courville, Christopher~Joseph Pal, and Chin-Wei Huang.
\newblock Ar-dae: Towards unbiased neural entropy gradient estimation.
\newblock \emph{ArXiv}, abs/2006.05164, 2020.

\bibitem[Robbins(1955)]{robbins1955empirical}
Herbert Robbins.
\newblock \emph{An empirical Bayes approach to statistics}.
\newblock Office of Scientific Research, US Air Force, 1955.

\bibitem[Miyasawa(1961)]{miyasawa1961empirical}
Koichi Miyasawa.
\newblock An empirical bayes estimator of the mean of a normal population.
\newblock \emph{Bull. Inst. Internat. Statist}, 38\penalty0 (181-188):\penalty0
  1--2, 1961.

\bibitem[Saremi and Hyvarinen(2019)]{saremi2019neural}
Saeed Saremi and Aapo Hyvarinen.
\newblock Neural empirical bayes.
\newblock \emph{Journal of Machine Learning Research}, 20:\penalty0 1--23,
  2019.

\bibitem[Kadkhodaie and Simoncelli(2020)]{kadkhodaie2020solving}
Zahra Kadkhodaie and Eero~P Simoncelli.
\newblock Solving linear inverse problems using the prior implicit in a
  denoiser.
\newblock \emph{arXiv preprint arXiv:2007.13640}, 2020.

\bibitem[Krizhevsky et~al.(2009)Krizhevsky, Hinton,
  et~al.]{krizhevsky2009learning}
Alex Krizhevsky, Geoffrey Hinton, et~al.
\newblock Learning multiple layers of features from tiny images.
\newblock Technical report, Citeseer, 2009.

\bibitem[Zhou et~al.(2019)Zhou, Gordon, Krishna, Narcomey, Fei-Fei, and
  Bernstein]{zhou2019hype}
Sharon Zhou, Mitchell Gordon, Ranjay Krishna, Austin Narcomey, Li~F Fei-Fei,
  and Michael Bernstein.
\newblock Hype: A benchmark for human eye perceptual evaluation of generative
  models.
\newblock In \emph{Advances in Neural Information Processing Systems}, pages
  3449--3461, 2019.

\bibitem[Zhang et~al.(2012)Zhang, Zhang, Mou, and
  Zhang]{zhang2012comprehensive}
Lin Zhang, Lei Zhang, Xuanqin Mou, and David Zhang.
\newblock A comprehensive evaluation of full reference image quality assessment
  algorithms.
\newblock In \emph{2012 19th IEEE International Conference on Image
  Processing}, pages 1477--1480. IEEE, 2012.

\bibitem[Zhao et~al.(2016)Zhao, Gallo, Frosio, and Kautz]{zhao2016loss}
Hang Zhao, Orazio Gallo, Iuri Frosio, and Jan Kautz.
\newblock Loss functions for image restoration with neural networks.
\newblock \emph{IEEE Transactions on computational imaging}, 3\penalty0
  (1):\penalty0 47--57, 2016.

\bibitem[Mao et~al.(2017)Mao, Li, Xie, Lau, Wang, and Smolley]{LSGAN}
Xudong Mao, Qing Li, Haoran Xie, Raymond~YK Lau, Zhen Wang, and Stephen~Paul
  Smolley.
\newblock Least squares generative adversarial networks.
\newblock In \emph{2017 IEEE International Conference on Computer Vision
  (ICCV)}, pages 2813--2821. IEEE, 2017.

\bibitem[Makhzani et~al.(2015)Makhzani, Shlens, Jaitly, Goodfellow, and
  Frey]{makhzani2015adversarial}
Alireza Makhzani, Jonathon Shlens, Navdeep Jaitly, Ian Goodfellow, and Brendan
  Frey.
\newblock Adversarial autoencoders.
\newblock \emph{arXiv preprint arXiv:1511.05644}, 2015.

\bibitem[Tolstikhin et~al.(2017)Tolstikhin, Bousquet, Gelly, and
  Schoelkopf]{tolstikhin2017wasserstein}
Ilya Tolstikhin, Olivier Bousquet, Sylvain Gelly, and Bernhard Schoelkopf.
\newblock Wasserstein auto-encoders.
\newblock \emph{arXiv preprint arXiv:1711.01558}, 2017.

\bibitem[Tran et~al.(2018)Tran, Bui, and Cheung]{tran2018dist}
Ngoc-Trung Tran, Tuan-Anh Bui, and Ngai-Man Cheung.
\newblock Dist-gan: An improved gan using distance constraints.
\newblock In \emph{Proceedings of the European Conference on Computer Vision
  (ECCV)}, pages 370--385, 2018.

\bibitem[Metz et~al.(2016)Metz, Poole, Pfau, and
  Sohl-Dickstein]{metz2016unrolled}
Luke Metz, Ben Poole, David Pfau, and Jascha Sohl-Dickstein.
\newblock Unrolled generative adversarial networks.
\newblock \emph{arXiv preprint arXiv:1611.02163}, 2016.

\bibitem[Yu et~al.(2015)Yu, Seff, Zhang, Song, Funkhouser, and
  Xiao]{yu2015lsun}
Fisher Yu, Ari Seff, Yinda Zhang, Shuran Song, Thomas Funkhouser, and Jianxiong
  Xiao.
\newblock Lsun: Construction of a large-scale image dataset using deep learning
  with humans in the loop.
\newblock \emph{arXiv preprint arXiv:1506.03365}, 2015.

\bibitem[LeCun et~al.(1998)LeCun, Bottou, Bengio, and
  Haffner]{lecun1998gradient}
Yann LeCun, L{\'e}on Bottou, Yoshua Bengio, and Patrick Haffner.
\newblock Gradient-based learning applied to document recognition.
\newblock \emph{Proceedings of the IEEE}, 86\penalty0 (11):\penalty0
  2278--2324, 1998.

\bibitem[Lin et~al.(2018)Lin, Khetan, Fanti, and Oh]{lin2018pacgan}
Zinan Lin, Ashish Khetan, Giulia Fanti, and Sewoong Oh.
\newblock Pacgan: The power of two samples in generative adversarial networks.
\newblock In \emph{Advances in neural information processing systems}, pages
  1498--1507, 2018.

\bibitem[Wenliang(2020)]{wenliang2020blindness}
Li~K Wenliang.
\newblock Blindness of score-based methods to isolated components and mixing
  proportions.
\newblock \emph{arXiv preprint arXiv:2008.10087}, 2020.

\bibitem[Radford et~al.(2015{\natexlab{b}})Radford, Metz, and Chintala]{DCGAN}
Alec Radford, Luke Metz, and Soumith Chintala.
\newblock Unsupervised representation learning with deep convolutional
  generative adversarial networks.
\newblock \emph{arXiv preprint arXiv:1511.06434}, 2015{\natexlab{b}}.

\bibitem[Dumoulin et~al.(2016)Dumoulin, Belghazi, Poole, Mastropietro, Lamb,
  Arjovsky, and Courville]{dumoulin2016adversarially}
Vincent Dumoulin, Ishmael Belghazi, Ben Poole, Olivier Mastropietro, Alex Lamb,
  Martin Arjovsky, and Aaron Courville.
\newblock Adversarially learned inference.
\newblock \emph{arXiv preprint arXiv:1606.00704}, 2016.

\bibitem[Srivastava et~al.(2017)Srivastava, Valkov, Russell, Gutmann, and
  Sutton]{srivastava2017veegan}
Akash Srivastava, Lazar Valkov, Chris Russell, Michael~U Gutmann, and Charles
  Sutton.
\newblock Veegan: Reducing mode collapse in gans using implicit variational
  learning.
\newblock In \emph{Advances in Neural Information Processing Systems}, pages
  3308--3318, 2017.

\bibitem[Lin et~al.(2017{\natexlab{b}})Lin, Khetan, Fanti, and Oh]{pacgan}
Zinan Lin, Ashish Khetan, Giulia Fanti, and Sewoong Oh.
\newblock Pacgan: The power of two samples in generative adversarial networks.
\newblock \emph{arXiv preprint arXiv:1712.04086}, 2017{\natexlab{b}}.

\bibitem[Kumar et~al.(2019)Kumar, Ozair, Goyal, Courville, and
  Bengio]{kumar2019maximum}
Rithesh Kumar, Sherjil Ozair, Anirudh Goyal, Aaron Courville, and Yoshua
  Bengio.
\newblock Maximum entropy generators for energy-based models.
\newblock \emph{arXiv preprint arXiv:1901.08508}, 2019.

\bibitem[Dieng et~al.(2019)Dieng, Ruiz, Blei, and Titsias]{dieng2019prescribed}
Adji~B Dieng, Francisco~JR Ruiz, David~M Blei, and Michalis~K Titsias.
\newblock Prescribed generative adversarial networks.
\newblock \emph{arXiv preprint arXiv:1910.04302}, 2019.

\bibitem[Kingma and Ba(2014)]{Adam}
Diederik~P Kingma and Jimmy Ba.
\newblock Adam: A method for stochastic optimization.
\newblock \emph{arXiv preprint arXiv:1412.6980}, 2014.

\bibitem[Lim and Ye(2017)]{lim2017geometric}
Jae~Hyun Lim and Jong~Chul Ye.
\newblock Geometric gan.
\newblock \emph{arXiv preprint arXiv:1705.02894}, 2017.

\bibitem[Jolicoeur-Martineau(2019)]{jolicoeur2019relativistic}
Alexia Jolicoeur-Martineau.
\newblock On relativistic $ f $-divergences.
\newblock \emph{arXiv preprint arXiv:1901.02474}, 2019.

\bibitem[Gidel et~al.(2019)Gidel, Hemmat, Pezeshki, Le~Priol, Huang,
  Lacoste-Julien, and Mitliagkas]{gidel2019negative}
Gauthier Gidel, Reyhane~Askari Hemmat, Mohammad Pezeshki, R{\'e}mi Le~Priol,
  Gabriel Huang, Simon Lacoste-Julien, and Ioannis Mitliagkas.
\newblock Negative momentum for improved game dynamics.
\newblock In \emph{The 22nd International Conference on Artificial Intelligence
  and Statistics}, pages 1802--1811, 2019.

\bibitem[Geirhos et~al.(2018)Geirhos, Rubisch, Michaelis, Bethge, Wichmann, and
  Brendel]{DBLP:journals/corr/abs-1811-12231}
Robert Geirhos, Patricia Rubisch, Claudio Michaelis, Matthias Bethge, Felix~A.
  Wichmann, and Wieland Brendel.
\newblock Imagenet-trained cnns are biased towards texture; increasing shape
  bias improves accuracy and robustness.
\newblock \emph{CoRR}, abs/1811.12231, 2018.
\newblock URL \url{http://arxiv.org/abs/1811.12231}.

\bibitem[Barratt and Sharma(2018)]{ISnote}
Shane~T. Barratt and Rishi Sharma.
\newblock A note on the inception score.
\newblock \emph{ArXiv}, abs/1801.01973, 2018.

\end{thebibliography}
